\documentclass[conference]{IEEEtran}

\usepackage[utf8]{inputenc}
\usepackage[T1]{fontenc}
\usepackage{url}                      
\usepackage{hyperref}
\usepackage{bm}
\usepackage{ifthen}
\usepackage{cite}
\usepackage{booktabs} 

\usepackage{amssymb}
\usepackage[cmex10]{amsmath} 

\usepackage{graphicx}
\usepackage{color}
\usepackage{algorithm,algorithmic}

\usepackage{amsfonts}

\usepackage{amsthm}
\usepackage{comment}
\newtheorem{thm}{Theorem}[section]

\newtheorem{lemma}[thm]{Lemma}

\newtheorem{prop}[thm]{Proposition}

\begin{document}

\title{Theoretical Interpretation of Learned Step Size in Deep-Unfolded Gradient Descent}

\author{%
  \IEEEauthorblockN{
  		Satoshi Takabe\IEEEauthorrefmark{1}\IEEEauthorrefmark{2} and Tadashi Wadayama\IEEEauthorrefmark{1}}
  \IEEEauthorblockA{\IEEEauthorrefmark{1}%
		Nagoya Institute of Technology,
		Gokiso, Nagoya, Aichi, 466-8555, Japan,\\
 		\{wadayama, s\_takabe\}@nitech.ac.jp} 
 \IEEEauthorblockA{\IEEEauthorrefmark{2}%
		RIKEN Center for Advanced Intelligence Project, Chuo-ku, Tokyo, 103-0027, Japan
 		} \\ 
}

\maketitle

\begin{abstract}
Deep unfolding is a promising deep-learning technique in which 
an iterative algorithm is unrolled to a deep network architecture with trainable parameters.
In the case of gradient descent algorithms,  
as a result of the training process, 
one often observes the acceleration of the convergence speed with 
learned non-constant step size parameters whose behavior 
is not intuitive nor interpretable from conventional theory.
In this paper, we provide a theoretical interpretation of the learned step size of 
deep-unfolded gradient descent (DUGD). 
We first prove that the training process of DUGD reduces not only the mean squared error loss 
but also the spectral radius related to the convergence rate.
Next, we show that minimizing the upper bound of the spectral radius naturally leads to 
the Chebyshev step which is a sequence of the step size based on Chebyshev polynomials.
The numerical experiments confirm that the Chebyshev steps qualitatively reproduce the learned step size parameters in DUGD, which  provides a plausible interpretation of the learned parameters.
Additionally, we show that the Chebyshev steps achieve the lower bound of the convergence rate for the first-order method in a specific limit without learning parameters {or} momentum terms.
\end{abstract}

\section{Introduction}
\label{intro}

Deep unfolding~\cite{gregor2010learning,hershey2014deep} is a promising deep learning approach
 whose architecture is based on existing iterative algorithms with tuning parameters such as step sizes in gradient descent (GD).
 The recursive structure of the algorithm is unrolled to a deep network and some parameters are embedded into the network. 
These parameters can be trained using standard deep learning techniques such as back propagation and 
stochastic GD if all the processes in the algorithm are differentiable.
{One notable advantage of deep unfolding is the acceleration of the convergence speed {that results from} tuning parameters compared with the original algorithm. 
Embedding proper trainable parameters also offers a flexible network structure to the algorithm that is applicable, for example, to inverse problems with/without prior information~\cite{monga2019algorithm}.}   
Since deep unfolding has been applied to iterative algorithms for compressed sensing~\cite{sprechmann2015learning,xin2016maximal,kamilov2016learning,borgerding2017amp,pmlr-v97-wu19b,ito2019trainable},
a number of deep unfolding-based algorithms have been proposed in various fields, such as image recovery~\cite{shi2017deep,jin2017deep,mardani2018neural,metzler2017learned,zhang2018ista,kellman2019physics}
 and wireless communications~\cite{nachmani2016learning,samuel2017deep,he2018model,8759948,yao2019sure,DBLP:conf/isit/WadayamaT19}.
 Recently, theoretical aspects of deep unfolding have also been investigated~\cite{chen2018theoretical,liu2018alista,mardani2019degrees}.

\begin{figure}[t]
   \centering
   \includegraphics[width=0.98\hsize]{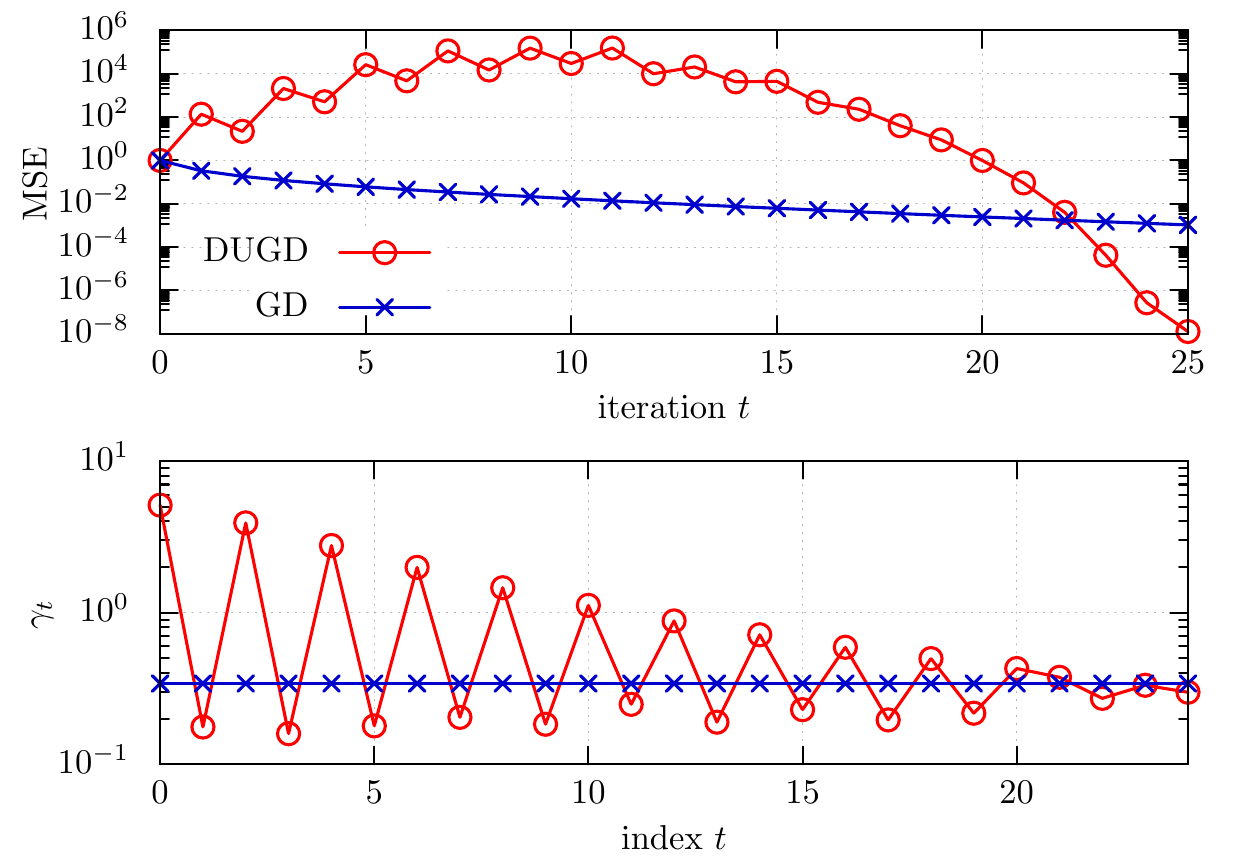}
    \caption{MSE performance (upper) and learned step size parameters $\{\gamma_t\}_{t=0}^{24}$ (lower) of DUGD (circles)
    and GD with a constant step size (cross marks) when $(n,m)=(300,600)$.
	The details of the experimental conditions are in Appendix A. 
}
    \label{fig_ga}
\end{figure}

To demonstrate deep unfolding,
we consider a simple least mean square (LMS) problem written as 
\begin{equation}
\bm{\hat \beta}:= \mbox{argmin}_{\bm \beta\in \mathbb{C}^n} \frac{1}{2}\|\bm y-\bm H \bm \beta\|_2^2,
\label{eq_LMS}
\end{equation}
where $\bm y :=\bm H \bm \beta + \bm n \in \mathbb{C}^m$ is the measurement vector with the noise vector {$\bm{n}$}  and measurement matrix $\bm{H}\in \mathbb{C}^{m\times n}$. 
Although the solution of (\ref{eq_LMS})  is explicitly given by $\bm{\hat \beta}=(\bm H^\ast\bm H)^{-1}\bm H^\ast\bm y$  ($m\ge n$) using the Hermitian transpose matrix $\bm H^\ast$,
GD is often used to reduce the computational complexity.
The recursive formula for GD is given by  
\begin{equation}
\bm{\beta}^{(t+1)}= \bm \beta^{(t)}{+}\gamma \bm H^\ast(\bm y - \bm H \bm \beta^{(t)}) \:(t=0,1,2,\dots),
\label{eq_GD1}
\end{equation}
where $\bm\beta^{(0)}$ is an initial vector and $\gamma$ is a step size parameter.
It is well known that the step size parameter {controls} the convergence speed of GD.
The optimal value of $\gamma$ is given by the largest and smallest eigenvalues of {$\bm H^\ast \bm H$} in the LMS problem, and it is found heuristically in general. 

Alternatively, we define deep-unfolded GD (DUGD) by 
\begin{equation}
\bm{\beta}^{(t+1)}= \bm \beta^{(t)}{+} \gamma_t \bm H^\ast(\bm y - \bm H \bm \beta^{(t)}),
\label{eq_GD2}
\end{equation}
where $\gamma_t$ is a trainable step size parameter that depends on the iteration index $t$.
The parameters $\{\gamma_t\}_{t=1}^T$ can be trained using training data {$\{(\bm{\tilde\beta}^{[k]}, \bm{\tilde y}^{[k]})\}_{k}$} by minimizing the loss function such as the mean squared error (MSE) {$\|\bm{\tilde \beta}-\bm\beta^{(T)}\|_2^2/n$} between the estimate $\bm{\beta}^{(T)}$ after $T$ iterations and {$\bm{\tilde \beta}$}. 
 Figure~\ref{fig_ga} shows the empirical results of the MSE performance (upper) and learned parameters $\{\gamma_t\}$ (lower) of DUGD and the original GD with the optimal constant step size when $(m,n)\!=\!(300,600)$ (see Appendix A for details).
We found that the learned parameter sequence had a \emph{zig-zag shape}, which accelerates the convergence speed compared with a naive GD with a constant step size. 
{These learned parameters are not intuitive or interpretable from conventional theory.}
This type of nontrivial learned step size parameters is observed not only for DUGD but also other deep-unfolded algorithms
that contain a nonlinear projection step~\cite{ito2019trainable,liu2018alista}.
Regarding an iterative soft thresholding algorithm for compressed sensing, it has been proved that large step sizes accelerate its convergence speed but {searching appropriate} step sizes is computationally difficult in practice~\cite{NIPS2019_9469}.

In this paper, the goal is to {provide a plausible interpretation of the learned parameters of DUGD and 
show that the parameters can accelerate the convergence speed of GD. }

The contributions of this paper are as follows:    

\begin{itemize}
\item {We show that minimizing the MSE loss in DUGD reduces the spectral radius related to 
 the convergence rate of GD. This suggests that appropriately learned step size parameters can improve the convergence rate.}
\item {
By minimizing the upper bound of the spectral radius, we derive 
Chebyshev steps that are a step size sequence based on Chebyshev polynomials.
Numerical experiments confirm} that the Chebyshev steps qualitatively reproduce the learned step size parameters in DUGD.
\item We perform convergence analysis of GD with the Chebyshev steps, which shows that the Chebyshev steps {improve} the convergence speed. Additionally, the convergence rate approaches the lower bound of first-order methods in a specific case, even though {it does not require a momentum term.  
The numerical results support the analysis, and we demonstrate an application to ridge regression.}
\end{itemize}

\textbf{Related works:} GD is a fundamental algorithm in continuous optimization~\cite{fletcher2013practical}.
GD was originally proposed by Cauchy~\cite{cauchy1847methode} and is known as the steepest descent algorithm.
The convergence rate of GD with a line search method that includes Cauchy's method is analyzed by Forsythe~\cite{forsythe1968asymptotic}. 
The acceleration of the convergence speed is a crucial issue {in the literature}. 
A well-known technique is the use of a momentum term. 
This originated from the heavy ball method~\cite{polyak1964some}, which is simply called the momentum method~\cite{1986Natur.323..533R} in the machine learning community.
Chebyshev semi-iterative method~\cite{Golub1989} and Nesterov's accelerated GD~\cite{nesterov1998introductory}
are other algorithms that use a momentum term. 
For convex quadratic problems, it has been proved that these algorithms with momentum terms are optimal because their convergence rates are proportional to the lower bound of first-order methods~\cite{nesterov1998introductory,lessard2016analysis}.
In this paper, we consider GD with a step size sequence and without momentum terms, which matches the recursive relation of DUGD.

\section{Deep-unfolded gradient descent}\label{sec_pre}
{In this paper}, we consider the minimization of a convex quadratic function 
$f(\bm{x}) = \bm{x}^T \bm{A} \bm{x}/2$ where $\bm A \in \mathbb{C}^{n \times n}$ is the Hermitian positive definite matrix and $\bm{x}_{\mathrm{opt}} = \bm{0}$ is its solution. {Note that this minimization problem corresponds to the LMS problem (\ref{eq_LMS}) under a proper transformation.}

The corresponding GD algorithm with the step size sequence $\{\gamma_t\}$ is given by 
\begin{equation}
	\bm{x}^{(t+1)} = (\bm{I}_n - \gamma_t \bm{A}) \bm{x}^{(t)} := \bm W^{(t)}\bm{x}^{(t)},\label{eq_W}
\end{equation}
where $\bm{I}_n$ is the identity matrix of order $n$ and  $\bm{x}^{(0)}$ is an arbitrary point in $\mathbb{C}^n$.

In DUGD, we first fix the total number of iterations $T(\ll n)$~\footnote{If $n<T$, GD always converges to the optimal solution {after $n$ iterations by setting} step sizes to the reciprocal of eigenvalues of $\bm A$. We thus omit this case. } and train the step size parameters $\{\gamma_t\}_{t=0}^{T-1}$.
Training these parameters is typically executed by minimizing the MSE loss function 
{$L(\bm x^{(T)}):=\|\bm x^{(T)}-\bm{x}_{\mathrm{opt}}\|_2^2/n$} between the output $\bm x^{(T)}$ of DUGD and the true solution $\bm{x}_{\mathrm{opt}} = \bm{0}$.
Additionally, to ensure the convergence of DUGD, we assume that DUGD uses a learned parameter sequence $\{\tilde\gamma_t\}_{t=0}^{T-1}$ repeatedly for $t>T$. 
Specifically, in the $t$th iteration, we assume that $\gamma_t := \tilde\gamma_{t'}$, where $t'\equiv t $ (mod $T$).
In this case, the {output} after every $T$ steps is written as
\begin{equation}
	\bm{x}^{((k+1)T)} = \left(\prod_{t=0}^{T-1}\bm W^{(t)}\right)\bm{x}^{(kT)} := \bm Q^{(T)}\bm{x}^{(kT)},
\label{eq_TPGT}
\end{equation}
for any $k=0,1,2,\dots$. Note that $\bm{Q}^{(T)}$ is a function of step size parameters $\{\gamma_t\}_{t=0}^{T-1}$.

Our motivation is to show that a proper step size parameter sequence $\{\gamma_t\}$ accelerates the convergence speed of GD. 
In this setup, an asymptotic convergence speed {with respect to the error between an estimate and the optimal solution} can be measured using the spectral radius of  a matrix $\bm Q^{(T)}$.
Let $\tau_1,\dots, \tau_n$ be the eigenvalues of  the matrix $\bm Q \in\mathbb{C}^{n\times n}$.
Then, the spectral radius of $\bm Q$ is defined as 
\begin{equation}
	\rho(\bm{Q}) := \max_{i\in\{1,\dots,n\}} \{|\tau_i| \}.
\end{equation}
For GD defined by (\ref{eq_TPGT}), it converges to the optimal solution if $\rho(\bm Q^{(T)})<1$ holds.
Additionally, because the error between $\bm{x}^{(kT)}$ and the optimal solution is bounded using $\rho(\bm Q^{(T)})$, the spectral radius indicates the asymptotic convergence rate of the algorithm.  

\section{Theoretical analysis}\label{sec_ana}

In this section, we show the following three facts: (i) minimizing the MSE loss in DUGD also reduces the spectral radius $\rho(\bm Q^{(T)})$, (ii) the step size parameter sequence defined by the explicit form minimizes the 
upper bound of the spectral radius $\rho(\bm Q^{(T)})$, and (iii) its convergence rate is smaller than a naive GD with a constant step size and asymptotically approaches the lower bound of the first order method.
These facts suggest that DUGD possibly accelerates the convergence speed by tuning step sizes properly.  

\subsection{Spectral radius and loss minimization}

The training process of deep unfolding consists of minimizing a loss function.
We show that minimizing a typical MSE loss also reduces the spectral radius $\rho(\bm Q^{(T)})$.

Before describing this claim, we first show the relation of $\rho(\bm Q^{(T)})$ to the eigenvalues of $\bm A$.
Recall that the Hermitian positive definite matrix $\bm A$ has $n$ positive eigenvalues including degeneracy.
{Hereafter, we assume that $\lambda_1\neq\lambda_n$ to avoid a trivial case.}

\begin{lemma}\label{lem_rho}
Let $\{\lambda_i\}_{i=1}^{n}$ be an eigenvalue sequence of $\bm{A}$ satisfying
 $(0<)\lambda_1\le \lambda_2\le \dots\le \lambda_n$. 
Then, we have
\begin{equation}
\rho(\bm Q^{(T)})=\max_{i=1,\dots,n}\left|\prod_{t=0}^{T-1}(1-\gamma_t\lambda_i)\right|.
\label{eq_rho}
\end{equation}
\end{lemma}

\begin{proof}
This is directly derived from (\ref{eq_W}) and the following fact: for a polynomial $p(x)$ with complex coefficients, if $\lambda$ is an eigenvalue of $\bm A$ associated with the eigenvector $\bm x$, then $p(\lambda)$ is an eigenvalue of the matrix $p(\bm A)$ associated with the eigenvector $\bm x$~\cite[p. 4-11, 39]{hogben2013handbook}.
\end{proof}

Using this lemma, we have the following theorem.
   
\begin{thm}\label{thm_spec}
Let $\bm{x}^{(0)}\in\mathbb{C}^n$ be a random variable over an isotropic probability density function $p(\bm{x}^{(0)})$ satisfying $0<\mathsf{E}_{\bm x^{(0)}}\|\bm x^{(0)}\|_2^2<\infty$.
Then, for any $T\in\mathbb{N}$, there exists a positive constant $C$ satisfying
\begin{equation}
{\rho(\bm Q^{(T)})}
\le C \sqrt{n\mathsf{E}_{\bm x^{(0)}}L(\bm x^{(T)})}.
\label{eq_mse}\end{equation}
\end{thm}

The details of the proof are in Appendix B.
{This theorem claims that minimizing the MSE loss function in DUGD reduces
 the corresponding spectral radius of $\bm Q^{(T)}$, which implies that appropriately learned step size parameters can accelerate the convergence speed of DUGD.}

\subsection{Chebyshev step}
{In this subsection, our aim is to {determine} a step size sequence that reduces the spectral radius to understand the nontrivial step size sequence of DUGD.
When $T\ge 2$, minimizing $\rho(\bm Q^{(T)})$ with respect to $\{\gamma_t\}_{t=0}^{T-1}$ is a {non-convex problem} in general. 
We alternatively} introduce a step size parameter sequence that bounds the spectral radius from above.

We first recall a well-known result when $T=1$, that is, a constant step size case~\cite[Section 1.3]{9781886529052}.
\begin{prop}\label{prop_fix}
Let  $\lambda_1(>0)$ and  $\lambda_n$ be the minimum and maximum eigenvalues of $\bm{A}$, respectively.
When $T=1$, the step size parameter that minimizes $\rho(\bm Q^{(1)})$  is given by
\begin{equation}
\gamma^\ast := \frac{2}{\lambda_1+\lambda_n}. \label{eq_single}
\end{equation} 
\end{prop}

{In the general case in which $T\ge 2$,
we focus on the step size sequence that minimizes the upper bound $\rho^{\mathrm{upp}}(\bm Q^{(T)})$ of the spectral radius $\rho(\bm Q^{(T)})$. The upper bound is} given by 
\begin{align}
\rho(\bm Q^{(T)})&=\max_{i=1,\dots,n}\left|\prod_{t=0}^{T-1}(1-\gamma_t\lambda_i)\right|\nonumber\\
&\le  \max_{ \lambda\in [\lambda_1,\lambda_n] }\left|\prod_{t=0}^{T-1}(1-\gamma_t\lambda)\right|:= \rho^{\mathrm{upp}}(\bm Q^{(T)}).\label{eq_upp}
\end{align}
Note that this upper bound is commonly analyzed~\cite[Section 3.4]{mason2002chebyshev}.

Next, we introduce a step size parameter sequence called Chebyshev steps which minimizes the above upper bound.
\begin{thm}
Let  $\lambda_1(>0)$ and  $\lambda_n$ be the minimum and maximum eigenvalues of $\bm{A}$, respectively.
For a given $T\in\mathbb{N}$, we define \emph{Chebyshev steps} $\{\gamma_t\}_{t=0}^{T-1}$ \emph{of length $T$} as
\begin{equation}
\gamma_t := \left[\frac{\lambda_n+\lambda_1}{2}+\frac{\lambda_n-\lambda_1}{2}\cos\left(\frac{2t+1}{2T}\pi\right)\right]^{-1}.
\label{eq_chstep}
\end{equation} 
Then, the Chebyshev steps is a sequence that minimizes the upper bound $\rho^{\mathrm{upp}}(\bm Q^{(T)})$ of spectral radius of $\bm Q^{(T)}$.
\end{thm}

Note that the Chebyshev step {is identical} to the optimal constant step size in Proposition~\ref{prop_fix}  when $T=1$.

We describe a sketch of the proof. The complete version is available in Appendix C.
The function $\prod_{t=0}^{T-1} (1-\gamma_t \lambda)$ {with the Chebyshev steps $\{\gamma_t\}_{t=0}^{T-1}$} is represented by a Chebyshev polynomial  $C_T(x)$ of order $T$. 
Using the minimax property that {$2^{1-T}C_T(x)$} is a monic polynomial that minimizes the 
$\ell_\infty$-norm in the Banach space $B[-1,1]$~\cite[Col. 3.4B]{mason2002chebyshev},
we can prove that the Chebyshev steps minimize $\rho^{\mathrm{upp}}(\bm Q^{(T)})$. 

The reciprocal of the Chebyshev steps $z_t:=\gamma_t^{-1}$ corresponds to Chebyshev points, that is, the zeros of {the shifted Chebyshev polynomial of order $T$ defined on $[\lambda_1,\lambda_n]$.}
Figure~\ref{fig_poi} shows the Chebyshev points and Chebyshev steps when $T=7$, $\lambda_1=1$, and $\lambda_n=9$.
A Chebyshev point is defined as a point that is projected onto an axis from a point of degree $\theta_t=(2t+1)\pi/(2T)$ on a semi-circle (see right part of Figure~\ref{fig_poi}). 
The Chebyshev points are located symmetrically with respect to the center of the circle corresponding to $(\gamma^\ast)^{-1}=(\lambda_1+\lambda_n)/2$.
The Chebyshev steps are given by $\gamma_t=z_t^{-1}$, which is shown in the left part of Figure~\ref{fig_poi}.
We found that the Chebyshev steps are widely located 
compared with the optimal constant step size $\gamma^\ast=1/5$.

\begin{figure}[t]
   \centering
   \includegraphics[width=0.98\hsize]{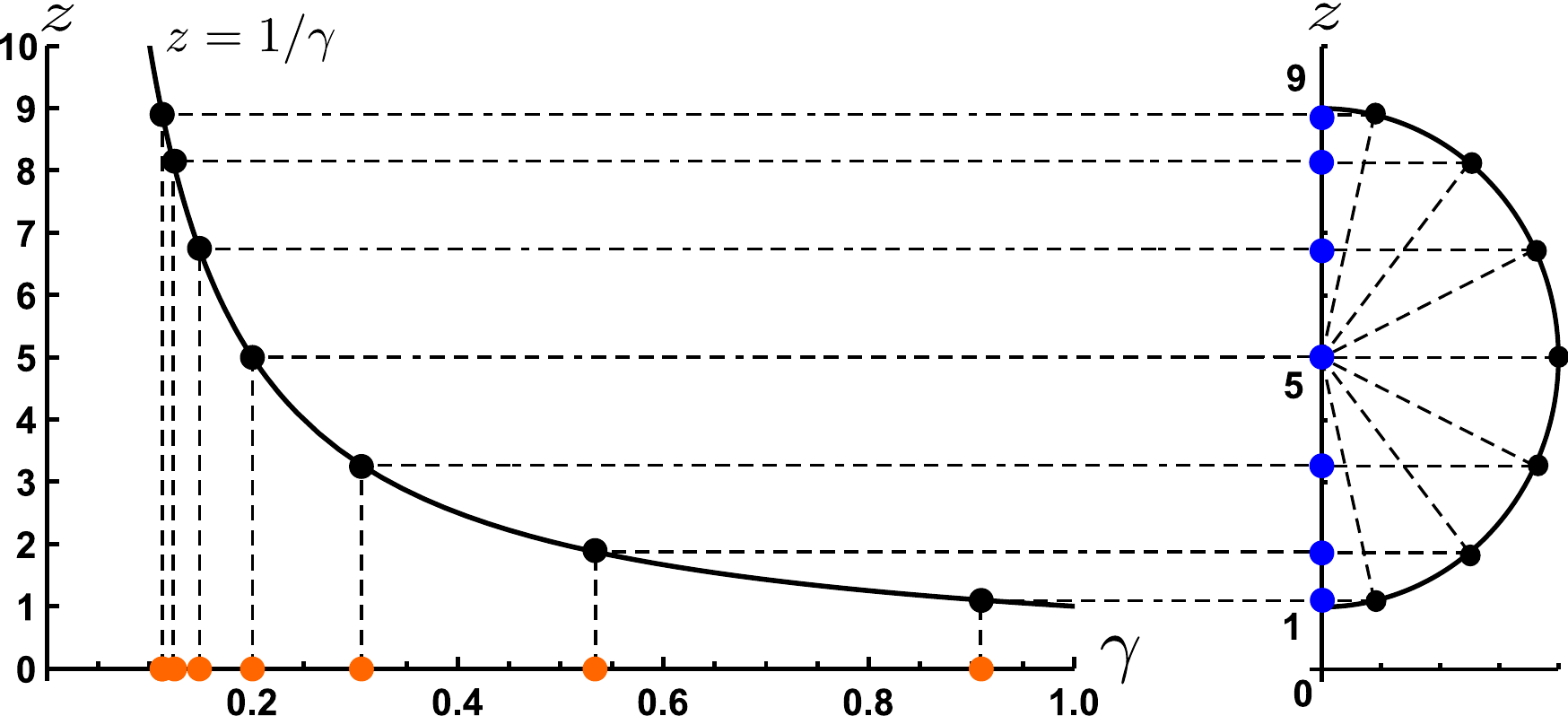}
    \caption{Chebyshev points $\{z_t\}_{t=0}^{T-1}$ (right; blue) and corresponding Chebyshev steps $\{\gamma_t\}_{t=0}^{T-1}$ (left; orange) when $T=7$, $\lambda_1=1$, and $\lambda_n=9$. 
}
    \label{fig_poi}
\end{figure}

\subsection{Convergence analysis}\label{sec_conv}

For convergence analysis, we show that {GD with the Chebyshev steps converges to the optimal solution.
Let $\bm Q^{(T)}_{\mathrm{Ch}}$ be the matrix $\bm{Q}^{(T)}$ with the Chebyshev steps of length $T$.}

\begin{prop}\label{prop_conv}
For any $k=0,1,2,\dots$ and ${T}\in\mathbb{N}$, we have
\begin{equation}
\|\bm x^{( (k+1)T)}-\bm{x}_{\mathrm{opt}} \|_2 \le   \rho^{\mathrm{upp}}(\bm Q^{(T)}_{\mathrm{Ch}}) \|\bm x^{( kT)}-\bm{x}_{\mathrm{opt}} \|_2.\label{eq_err}
\end{equation}
\end{prop}

\begin{proof}
{Let $\|\bm A\|_{\mathrm{op}}:=\sup_{\|\bm v\|_2=1} \|\bm A\bm v\|_2$ be an operator norm of $\bm{A}$.
Because $\bm Q^{(T)}_{\mathrm{Ch}}$ is a normal matrix, $\|\bm Q^{(T)}_{\mathrm{Ch}}\|_{\mathrm{op}}=\rho(\bm Q^{(T)}_{\mathrm{Ch}})$ holds. }
Using (\ref{eq_TPGT}), (\ref{eq_upp}), and $\bm{x}_{\mathrm{opt}}=\bm 0$, we have
\begin{align}
\|\bm x^{( (k+1)T)} \|_2 
&= \|\bm Q^{(T)}_{\mathrm{Ch}} \bm x^{( kT)} \|_2 \nonumber\\
&\le \|\bm Q^{(T)}_{\mathrm{Ch}}\|_{\mathrm{op}}  \|\bm x^{( kT)} \|_2 \nonumber\\
&= \rho(\bm Q^{(T)}_{\mathrm{Ch}})  \|\bm x^{( kT)} \|_2 \nonumber\\
&\le \rho^{\mathrm{upp}}(\bm Q^{(T)}_{\mathrm{Ch}}) \|\bm x^{( kT)} \|_2.
\end{align}
\end{proof}

The main claim in this subsection is that the Chebyshev steps of length $T(\ge 2)$ accelerate the convergence speed with respect to a spectral radius.

\begin{thm}\label{thm_main}
Let $\bm Q^{(T)}_{\mathrm{ch}}$ be the matrix $\bm Q^{(T)}$ with the Chebyshev steps of length $T(\ge 2)$.
We also define $\bm Q^{(T)}_{\mathrm{s}}$ as $\bm Q^{(T)}$ with the optimal constant step size, that is, $\gamma_0=\dots=\gamma_{T-1}=\gamma^\ast$.
Then, we have
\begin{equation}
\rho(\bm{Q}^{(T)}_{\mathrm{ch}}) < \rho(\bm Q^{(T)}_{\mathrm{s}}). \label{eq_ch_rho}
\end{equation}
\end{thm}


The proof of this theorem is divided into two parts. First, using (\ref{eq_upp}) and the definition of the Chebyshev steps, we show that the spectral radius $\rho(\bm{Q}^{(T)}_{\mathrm{ch}})$
is bounded by 
\begin{align}
\rho(\bm{Q}^{(T)}_{\mathrm{ch}})&\le
\rho^{\mathrm{upp}}(\bm Q^{(T)}_{\mathrm{Ch}})\nonumber\\
& = \left\{\frac{1}{2}\left[ \left(\frac{\sqrt\kappa+1}{\sqrt\kappa-1}\right)^T
+\left(\frac{\sqrt\kappa-1}{\sqrt\kappa+1}\right)^T\right]\right\}^{-1},\label{eq_br}
\end{align}
where $\kappa:=\lambda_n/\lambda_1$ is the condition number of the matrix $\bm A$.
Second, we prove that $\rho^{\mathrm{upp}}(\bm Q^{(T)}_{\mathrm{Ch}})< \rho(\bm Q^{(T)}_{\mathrm{s}}) = [(\kappa-1)/(\kappa+1)]^T$. Further details are in Appendix D.

The convergence rate of GD is defined as $R:=\liminf_{t\rightarrow\infty} \rho(\bm Q^{(t)})^{1/t}$.
From (\ref{eq_br}), the convergence rate $R_{\mathrm{CHGD}}(T)$ of GD with the Chebyshev steps (CHGD) of length $T$
 is bounded by
\begin{equation}
R_{\mathrm{CHGD}}(T) \le\left\{\frac{1}{2}\left[ \left(\frac{\sqrt\kappa+1}{\sqrt\kappa-1}\right)^T
+\left(\frac{\sqrt\kappa-1}{\sqrt\kappa+1}\right)^T\right]\right\}^{-\frac{1}{T}}, \label{eq_ch_rho}
\end{equation}
which is lower than the convergence rate of GD with the optimal constant step size 
$R_{\mathrm{S}} = (\kappa-1)/(\kappa+1)$.
The rate $R_{\mathrm{CHGD}}(T)$ approaches the well-known lower bound of first order methods from above, which is given by
\begin{equation}
R_{\mathrm{low}}:=\frac{\sqrt\kappa-1}{\sqrt\kappa+1}. \label{eq_ch_low}
\end{equation}
This bound is strict because some GD algorithms, such as the momentum method and Nesterov acceleration, can achieve this rate in the strongly convex case.
In Figure~\ref{fig_poi}, we show the convergence rates of GD and CHGD as functions of $\kappa$.
{For  CHGD, some upper bounds of the convergence rates with different $T$ are plotted.}
We confirmed that CHGD with $T\ge 2$ has a smaller convergence rate than GD with the optimal constant step size and the rate of CHGD approaches the lower bound (\ref{eq_ch_rho}) quickly as $T$ increases.
All rates and the lower bound converge to $1$ in the large-$\kappa$ limit.

\begin{figure}[t]
   \centering
   \includegraphics[width=0.98\hsize]{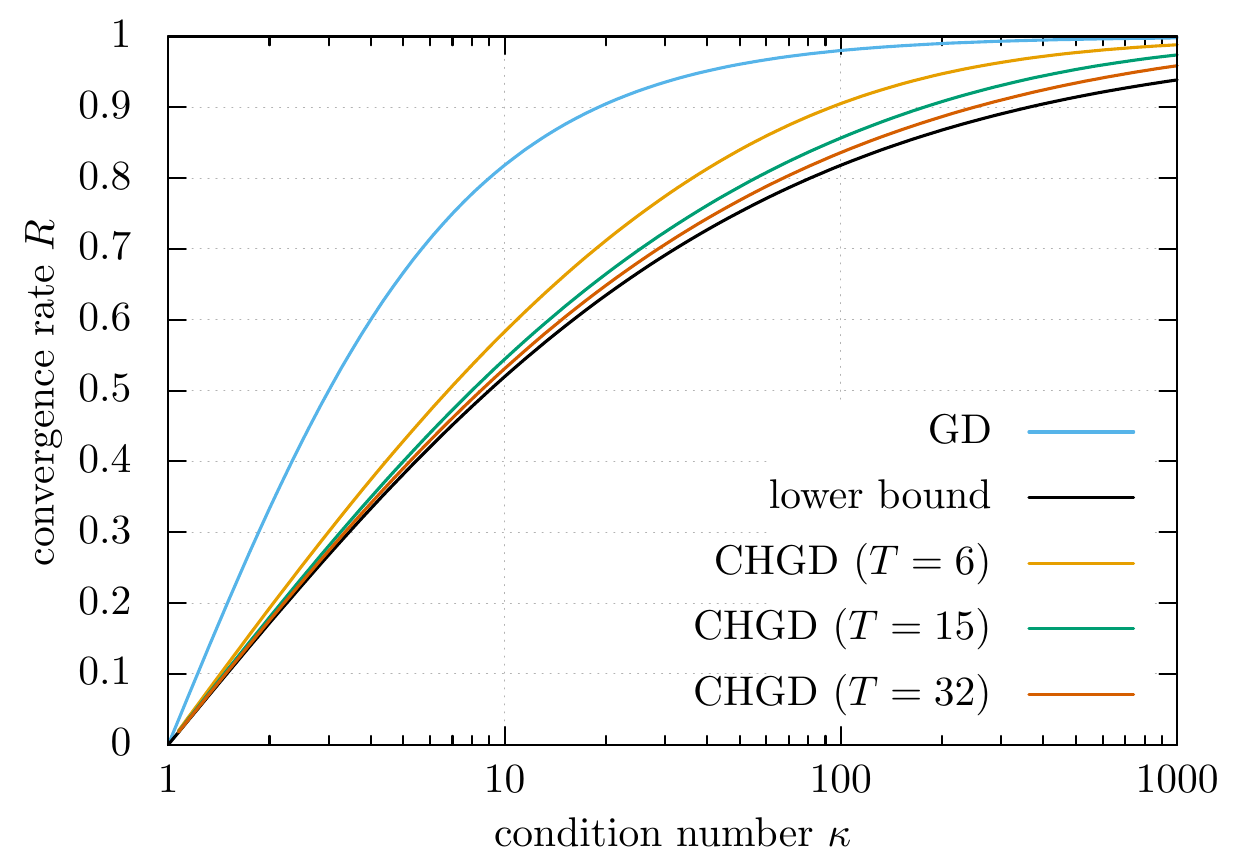}
    \caption{Comparison of convergence rates as functions of condition number $\kappa$. 
}
    \label{fig_poi}
\end{figure}

To summarize, we focus on the fact that the training process of DUGD reduces the spectral radius $\rho(\bm{Q}^{(T)})$
and introduce learning-free Chebyshev steps that minimize its upper bound and improve the convergence rate.

\section{Numerical comparison of DUGD}

In this section, we examine DUGD following the setup in Sec~\ref{sec_pre}. 
The main goal is to examine whether Chebyshev steps explain the nontrivial learned step size parameters in DUGD.
Additionally, we also analyze the convergence property of DUGD and CHGD numerically.

\subsection{Experimental conditions}
We describe the details of the numerical experiments.
DUGD was implemented using PyTorch 1.3~\cite{NIPS2019_9015}. 
Each training datum was given by a pair of the random initial point $\bm{x}^{(0)}\in\mathbb{R}^n$ and optimal solution $\bm{x}_{\mathrm{opt}} = \bm{0}$.
The random initial point was generated as the i.i.d. Gaussian random vector with unit mean and unit variance.
The matrix $\bm{A}$ is generated by
{$\bm{A}\!=\!\bm{H}^T\bm{H}$ with the random Gaussian matrix $\bm{H}\in\mathbb{R}^{m\times n}$} whose elements were i.i.d.  Gaussian random variables with zero mean and variance $1/n$.
Then, the eigenvalue distribution of $\bm A$ followed the Marchenko-Pastur distribution 
as $n\rightarrow\infty$ with $m/n$ fixed to a constant. The maximum and minimum eigenvalues approach 
$(1+\sqrt{m/n})^2$ and $(1-\sqrt{m/n})^2$, respectively.
The matrix $\bm{A}$ was fixed throughout the training process.
  
The training process was executed using incremental training in which we gradually increased the number of layers (iterations $T$) by initializing the value of the parameter $\gamma_t$ ($t=0,\dots T-1$) to a learned value in the former training process (generation)~\cite{metzler2017learned,ito2019trainable}. 
{Incremental training can improve the performance of DUGD compared with conventional one-shot training in which all layers are trained at once.
At the beginning of the training process, all initial values of  $\{\gamma_t\}$ were set to $0.3$ unless otherwise noted. }
For each generation, the parameters were optimized to minimize the MSE loss function $L(\bm{x}^{(T)})$ between the  output of  DUGD and the optimal solution using  $500$ mini batches of size $200$.
We used Adam optimizer~\cite{kingma2014adam} with the learning rate $0.002$.

\subsection{Spectral radius and MSE loss}

We numerically verified the relation between the MSE loss in DUGD and the corresponding spectral radius $\rho(\bm Q^{(T)})$ {in Theorem~\ref{thm_spec}} up to $T\!=\!15$.
Figure~\ref{fig_loss} shows an example of the comparison when $(n,m)\!=\!(300,1200)$. 
To estimate the average MSE loss  $\mathsf{E}_{\bm x^{(0)}}L(\bm{x}^{(T)})$, we used the empirical MSE loss after the $T$th generation. 
The constant $C$ on the r.h.s. of (\ref{eq_mse}) was evaluated numerically (see Appendix B). 
We confirmed that the spectral radius $\rho(\bm Q^{(T)})$ was upper bounded by (\ref{eq_mse}) with the MSE loss $L(\bm{x}^{(T)})$.

\begin{figure}[t]
   \centering
   \includegraphics[width=0.98\hsize]{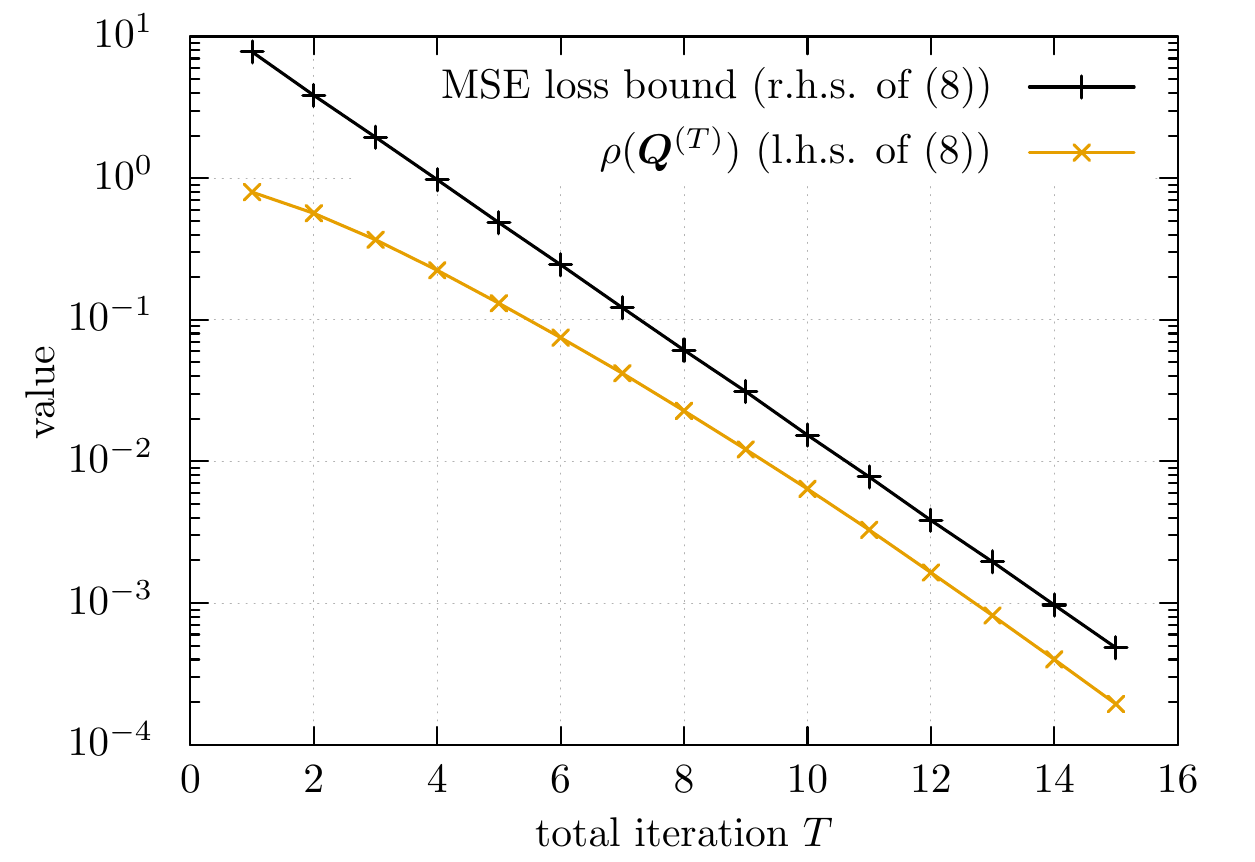}
    \caption{Comparison of the MSE loss bound in Theorem~\ref{thm_spec} and spectral radius $\rho(\bm Q^{(T)})$ in DUGD
     when $(n,m)=(300,1200)$. 
}
    \label{fig_loss}
\end{figure}

\subsection{Learned step sizes}

Next, we examined the learned step size parameter sequence in DUGD.
Figure~\ref{fig_step1} shows {examples of sequences of length $T=6$ and $15$ when $(n,m)\!=\!(300,1200)$}. 
 {To compare parameters directly, they} were rearranged in descending order, although the learned parameters indeed had a zig-zag shape.
The black symbols represent the Chebyshev steps with asymptotic maximal and minimal eigenvalues $\lambda_n=9$ and $\lambda_1=1$ when $m/n=4$.
 Other symbols indicate the learned step size parameters corresponding to five trials, that is, different matrices of $\bm A$ and the training process with different random seeds.
 We found that the learned step sizes agreed with each other, which indicates the self-average property of random matrices and success of the training process.
 More importantly, they were close to the Chebyshev steps, particularly when $\gamma$ was small.
 We found that,  when $T=6$, 
 the gap between the Chebyshev steps and learned step sizes was larger than in the $T=15$ case.

\begin{figure}[t]
   \centering
   \includegraphics[width=0.98\hsize]{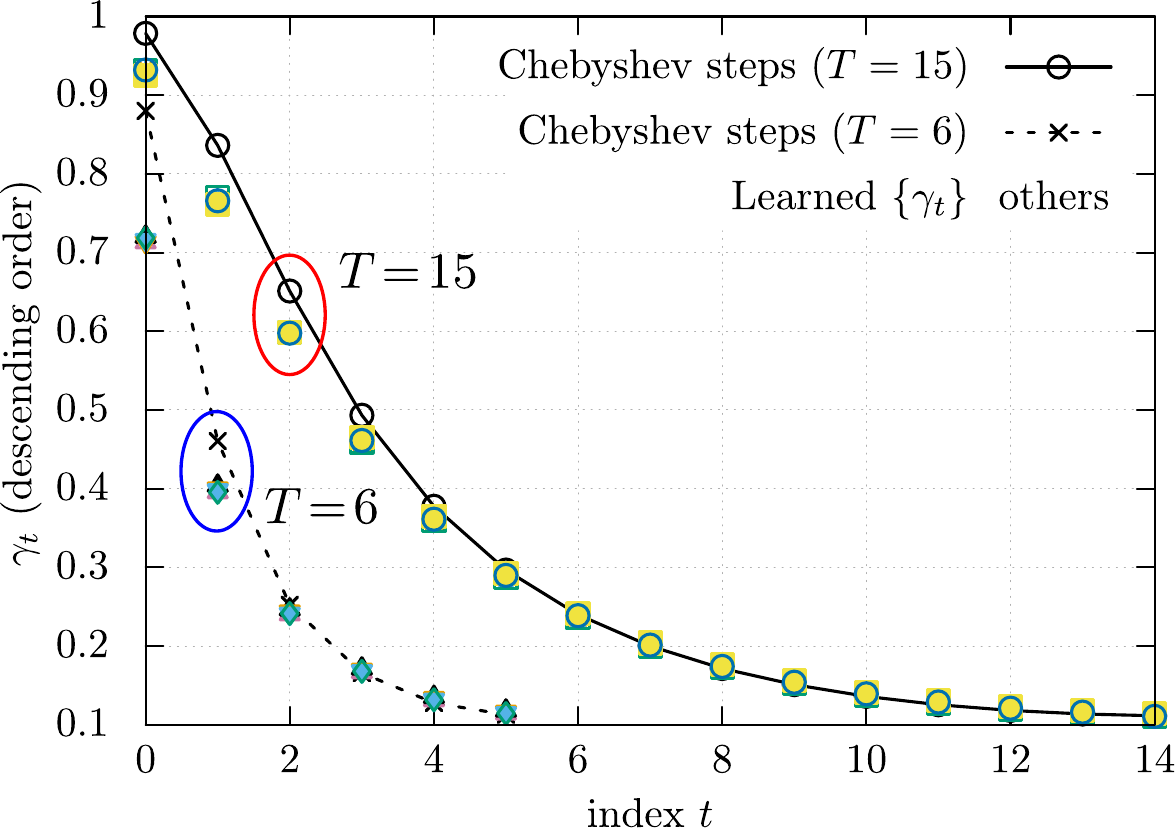}
    \caption{Chebyshev steps (black symbols) and learned step size parameters in DUGD (others; $5$ trials) in  descending order when $(n,m)=(300,1200)$ and $T=6$ (dotted) and $15$ (solid). 
}
    \label{fig_step1}
\end{figure}

{The zig-zag shape of the learned step size parameters is another nontrivial behavior of DUGD, although the order of the   parameters does not affect the MSE performance after the $T$th iteration. 
It is numerically suggested that the shape depends on the training process, particularly on incremental training and initial values of $\{\gamma_t\}$.
In fact, we can determine a permutation of the Chebyshev steps systematically by emulating the training process.
Figure~\ref{fig_step2} shows the learned step size parameters in DUGD ($T\!=\!11$) with different initial values of $\{\gamma_t\}$ and corresponding permuted Chebyshev steps.
We found that they agreed with each other including the order of parameters.
Further details are in Appendix E. 
}

\begin{figure}[t]
   \centering
   \includegraphics[width=0.98\hsize]{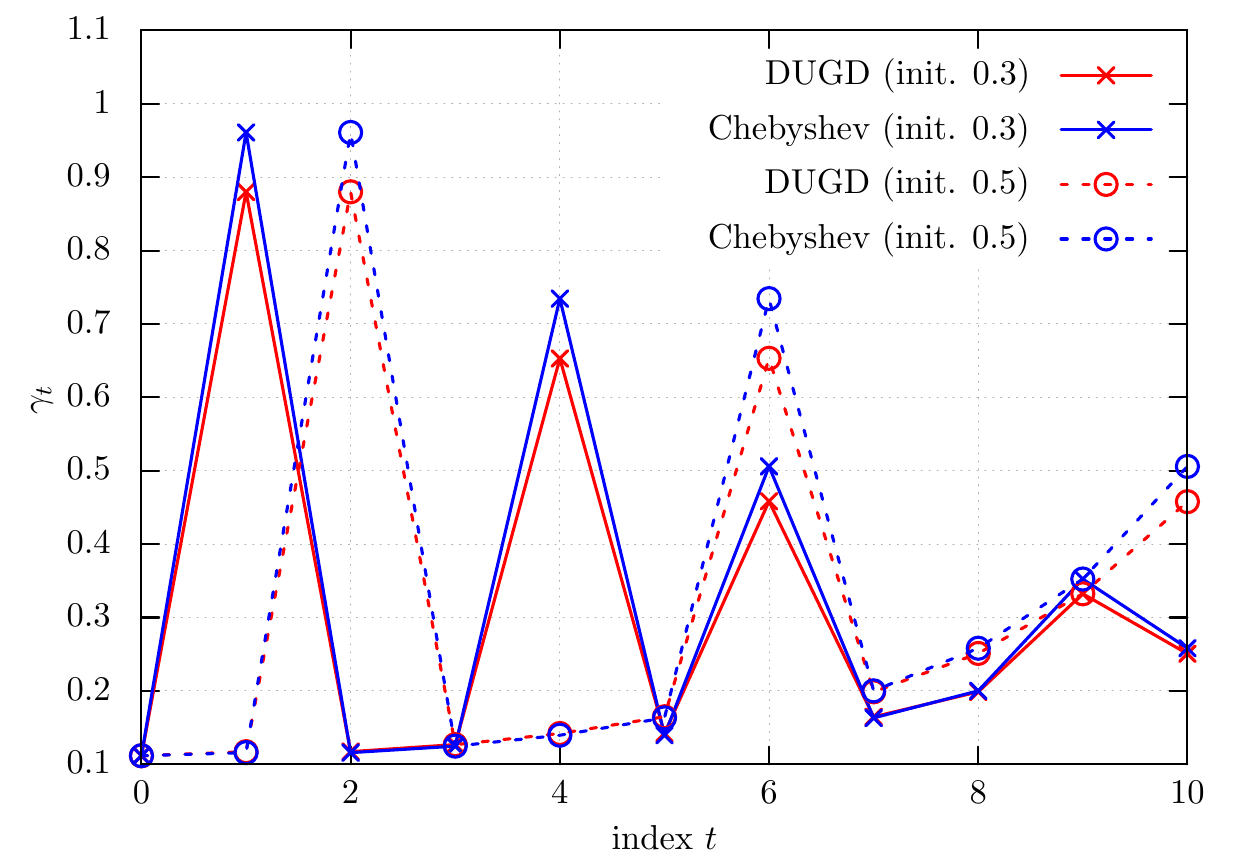}
    \caption{Zig-zag shape of the learned step size parameters (red) in DUGD and  permuted  Chebyshev steps (blue) with different initial values of $\gamma_t$ when $(n,m)=(300,1200)$ and $T=11$. 
}
    \label{fig_step2}
\end{figure}

\begin{figure}[t]
   \centering
   \includegraphics[width=0.98\hsize]{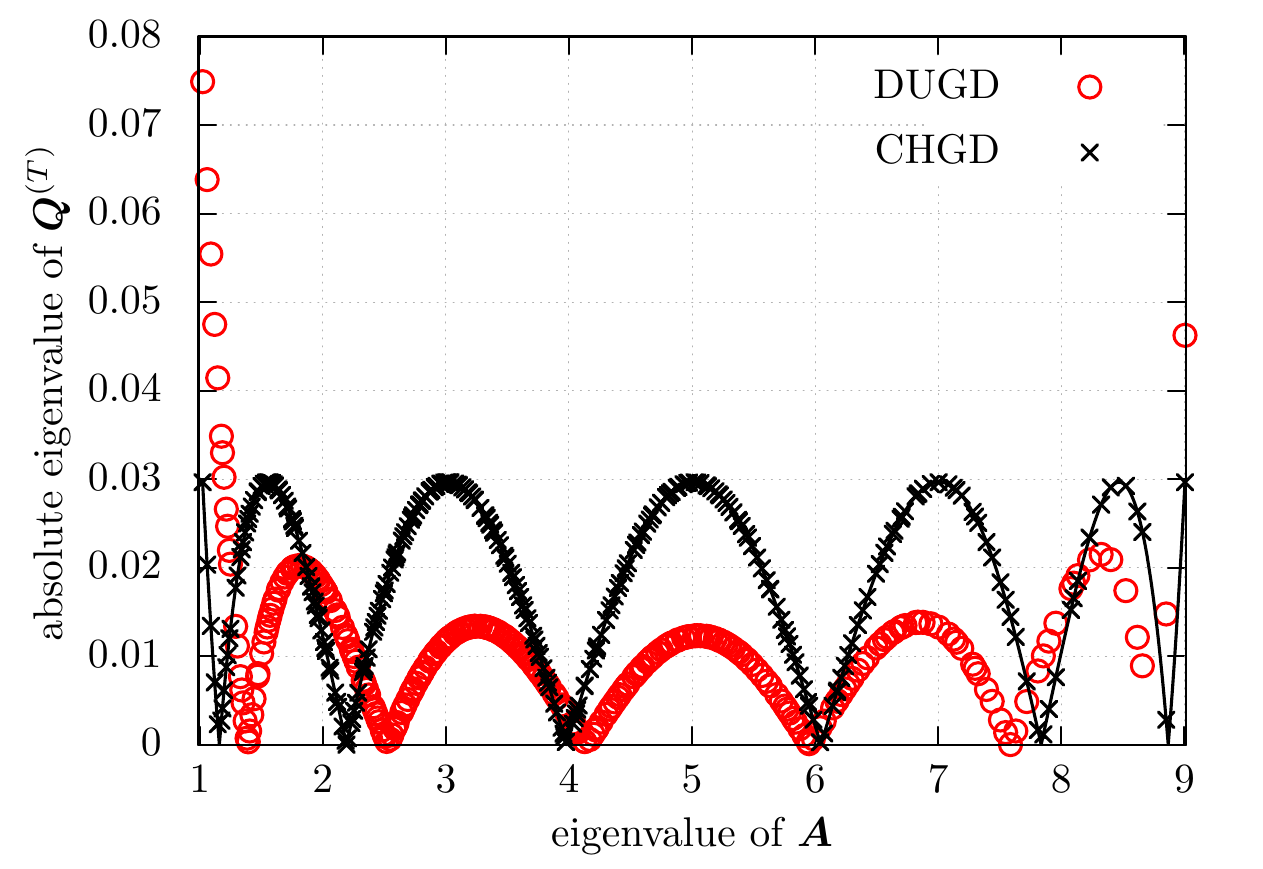}
    \caption{Absolute eigenvalues of $\bm{Q}^{(T)}$ using learned step size parameters (red) 
    as a function of the eigenvalues of $\bm{A}$ when $T=7$,  $\lambda_1=1.0$, and $\lambda_n=9.0$. 
    Black symbols represent the corresponding absolute eigenvalues when the Chebyshev steps are used.
}
    \label{fig_spec_c}
\end{figure}

 To understand the discrepancy between the Chebyshev steps and learned step sizes, we present the 
 absolute eigenvalues of $\bm{Q}^{(T)}$ as a function of the eigenvalues of $\bm{A}$ in Figure~\ref{fig_spec_c}.
 If the step size sequence of length $T$ is given by $\{\gamma_t\}_{t=0}^{T-1}$, then
 the absolute eigenvalues of $\bm{	Q}^{(T)}$ corresponding to the eigenvalue $\lambda$ of $\bm{A}$ are 
 written as
$|\tau(\lambda)| = \left|\prod_{t=0}^{T-1}(1-\gamma_t\lambda)\right|$.
Figure~\ref{fig_spec_c} shows $|\tau(\lambda)|$ when $\{\gamma_t\}_{t=0}^{T-1}$ is a learned step size parameter sequence (red) and Chebyshev steps (black) of  length $T=6$.
To show the spectral density, symbols are located at each eigenvalue $\lambda_i$ of matrix $\bm A$. 
We found that $\{|\tau(\lambda_i)|\}$ of the learned step sizes were smaller than
 those of  the Chebyshev steps in the high spectral-density regime and larger otherwise.
This is because it reduced the MSE loss that all the eigenvalues of matrix $\bm A$ contributed.
By contrast, it increased the maximum value of $|\tau(\lambda_i)|$ corresponding
 to the spectral radius $\rho(\bm Q^{(T)})$.
In this case, the spectral radius of DUGD was $0.074$ whereas that of the Chebyshev steps was $0.029$. 
Recalling that the spectral radius of the optimal constant step size $\gamma^\ast=1/5$ was $\rho(\bm Q_{\mathrm s}^{(T)})\simeq 0.262$,
DUGD accelerated the convergence speed in terms of the spectral radius whereas the
Chebyshev steps further improved the convergence rate.

\subsection{Performance analysis and convergence rate}

Finally, we examined the convergence performance of DUGD and {CHGD, and verified the convergence rate evaluated in Section~\ref{sec_conv}.}

\begin{figure}[t]
   \centering
   \includegraphics[width=0.98\hsize]{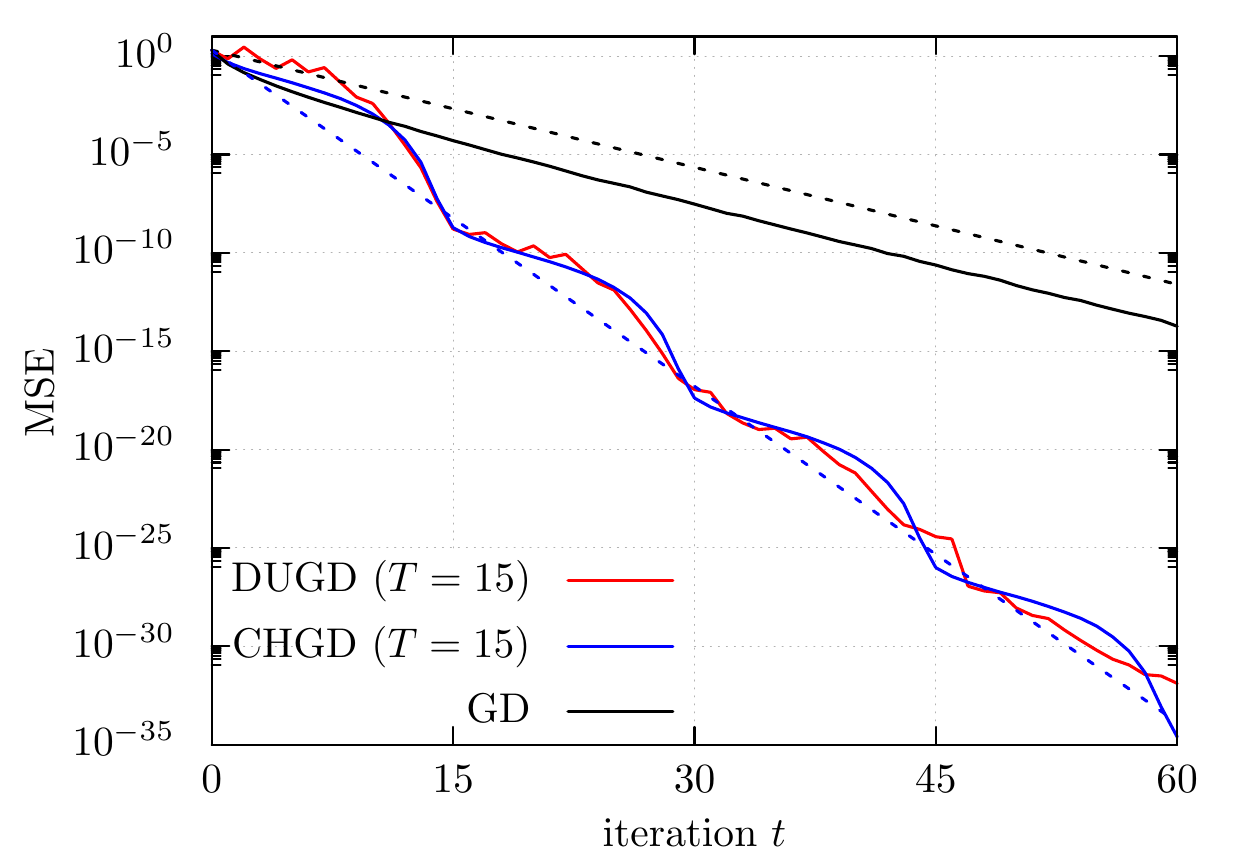}
    \caption{Comparison of the MSE performance of GD algorithms when $(n,m)=(300,1200)$ ($\kappa=8.79$).
    Dotted lines represent the slopes of the convergence rate of CHGD (\ref{eq_ch_rho}) and GD $R_S$.
}
    \label{fig_perf}
\end{figure}

In the experiment, 
the MSE of DUGD was evaluated as a generalization error {over $10^4$ random initial points}.  
In DUGD, we first trained the step sizes with $T=15$ and repeated them for every $15$ iterations.
Similarly, CHGD was executed repeatedly with the Chebyshev steps of length $15$.

{Figure~\ref{fig_perf} shows the MSE performance of DUGD, CHGD, and GD with the optimal constant step size when $(n,m)=(300,1200)$.
We found that DUGD and CHGD converged faster than GD, which shows that a proper step size parameter sequence accelerated the convergence speed.
Although DUGD had slightly better MSE performance than CHGD when $t=15$, 
CHGD exhibited faster convergence as the number of iterations increased.
This is because the spectral radius of CHGD was smaller than that of DUGD, as discussed in the last subsection.
Figure~\ref{fig_perf} also shows the MSE calculated by (\ref{eq_err}) using the convergence rates.
We found that the upper bound of the convergence rate (\ref{eq_ch_rho}) correctly predicted the convergence property of CHGD.}

{To summarize, we numerically verified the theoretical analyses in the last section, and found that the Chebyshev steps qualitatively reproduced a learned step size sequence in DUGD.
 This also indicates that deep unfolding can accelerate the GD algorithm by tuning its step size parameters.
}

\section{Application of Chebyshev steps}

{In this section, we consider an application of CHGD instead of DUGD because CHGD requires no training process.
After we compare CHGD with other accelerated GD algorithms, we demonstrate a practical application of CHGD to ridge regression.}

\subsection{Comparison with accelerated GD}

\begin{figure}[t]
   \centering
   \includegraphics[width=0.98\hsize]{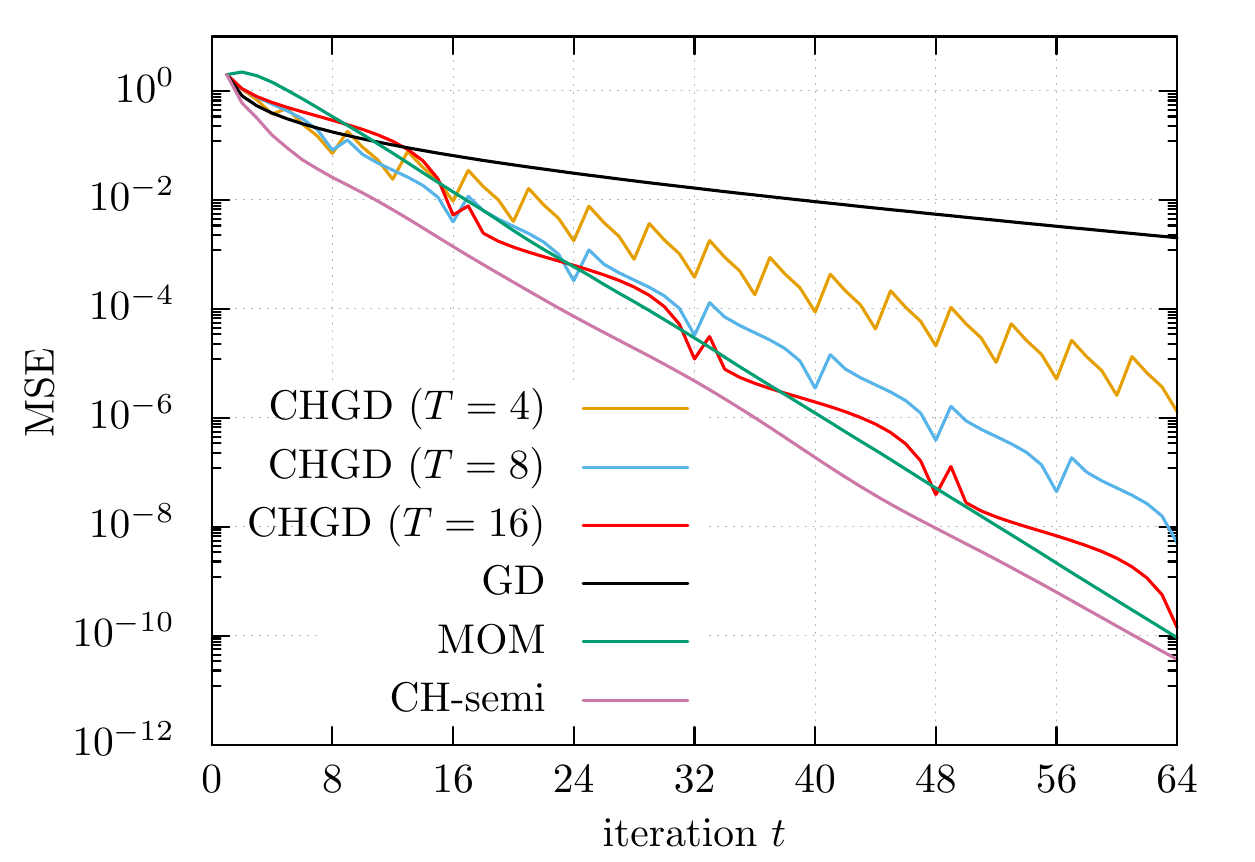}
    \caption{MSE performance of CHGD ($T\!=\!4,8,16$) and other GD algorithms when $(n,m)\!=\!(300,450)$ ($\kappa\!=\!88.1$, $10^4$ samples).
}
    \label{fig_comp}
\end{figure}

{We compared CHGD with two GD algorithms with a momentum term.
One is the momentum method (MOM) whose recursive relation is given by
\begin{equation}
	\bm{x}^{(t+1)} = (\bm{I}_n - \gamma' \bm{A}) \bm{x}^{(t)} +\beta(\bm{x}^{(t)}-\bm{x}^{(t-1)}),
	\label{eq_mom}
\end{equation}
where $\bm x^{(-1)}=\bm 0$, $\gamma':=4/(\sqrt\lambda_1+\sqrt\lambda_n)^2$, and 
$\beta:=((\sqrt{\kappa}-1)/(\sqrt{\kappa}+1))^2$~\cite{polyak1964some}. 
The other is the Chebyshev semi-iterative method (CH-semi) defined as
\begin{align}
	\bm{x}^{(t+1)} &= (\bm{I}_n - \gamma'_{t+1} \bm{A}) \bm{x}^{(t)} +(\gamma'_{t+1}-1) (\bm{x}^{(t)}-\bm{x}^{(t-1)}),\nonumber\\
	\gamma'_{t+1} &= \frac{4}{4-\xi^2\gamma'_t}\: (t\ge 2),	\label{eq_mom}
\end{align}
where $\bm x^{(-1)}=\bm 0$, $\gamma'_1=1$,  $\gamma'_2=2/(2-\xi^2)$, and $\xi = 1-1/\kappa$~\cite{Golub1989}.
These achieve the lower bound (\ref{eq_ch_low}) of the convergence rate.
}

{Figure~\ref{fig_comp} shows the MSE performance of CHGD ($T\!=\!4,8,16$) and other GD algorithms. 
We found that CHGD improved its MSE performance as $T$ increased. In particular, CHGD ($T=16$) had reasonable performance compared with MOM and CH-semi. 
This indicates that CHGD is an \emph{accelerated GD algorithm without momentum terms}.  
}

\subsection{Applications to ridge regression for real data}

We demonstrate CHGD for ridge regression.
Ridge regression, also known as Tikhonov regularization, is a fundamental biased estimation for ill-conditioned problems~\cite{hoerl1970ridge}. 
We consider a noisy linear observation $\bm y=\bm H \bm \beta+\bm n$ with a measurement matrix $\bm{H}\in \mathbb{R}^{m\times n}$.
When $m\gg n$, the LMS problem becomes ill-conditioned, which leads to numerical instability.
Instead, ridge regression is often used, which is defined as       
\begin{equation}
\bm{\hat \beta}:=\mbox{argmin}_{\bm \beta\in \mathbb{R}^n} \frac{1}{2}\|\bm y-\bm H \bm \beta\|_2^2
+\frac{\eta}{2}\|\bm \beta\|_2^2,
\label{eq_rid}
\end{equation}
where $\eta$ is a regularization coefficient controlling weight of the estimate $\bm{\hat{\beta}}$.
As the parameter $\eta$ reduces the condition number of the matrix $\bm H^T\bm H+\eta\bm  I_n$, 
a simple ridge estimator $(\bm H^T\bm H+\eta \bm I_n)^{-1}\bm H^T\bm y$ is available. 
However, it takes $O(n^3)$ computation time to calculate a
pseudo-inverse matrix.
This computational cost increases if we search a proper $\eta$ by sweeping its value.

An alternative approach to solve (\ref{eq_rid}) is to use a GD algorithm. 
Because it contains no inverse of the matrix, GD runs in $O(n^2)$ time. 
The drawback of GD is its slow convergence when $\eta$ is relatively small. In this sense, using a GD algorithm with faster convergence is important. 

To examine the convergence speed of GD in ridge regression, we performed the three algorithms: GD with the optimal constant step size, CHGD, and the momentum method  (\ref{eq_mom}).
As an example, ridge regression was applied to Communities and Crime Dataset~\cite{redmond2002data} in the UCI Machine Learning Repository~\cite{Dua:2019}. After removing elements containing missing values, we have $(n,m)=(98, 1994)$. The moment matrix $\bm{A}=\bm{H}^T\bm H$ had a huge condition number, $\kappa\simeq 7.8\times 10^5$, which indicates that the inverse problem is ill-conditioned.
In the experiments, all the algorithms used the maximum and minimum values of the matrix $\bm A$. 
Using the power method, we can estimate the maximum eigenvalue $\lambda_n$ in $O(n^2)$ time  
instead of computing eigenvalues directly in $O(n^3)$ time.  
For the estimation of the minimum eigenvalue, the power method is also applicable to a shifted matrix $\lambda_n\bm I_n-\bm A$. 

\begin{figure}[t]
   \centering
   \includegraphics[width=0.98\hsize]{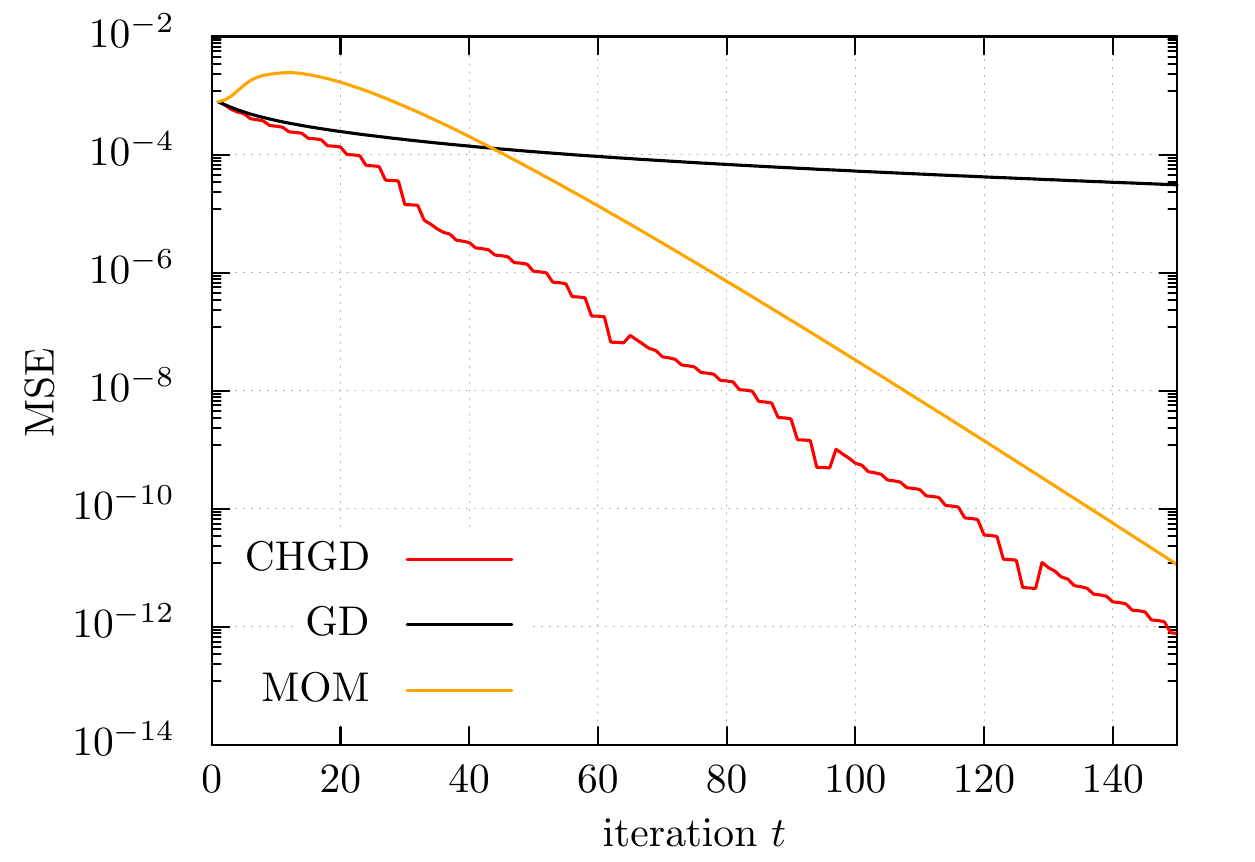}
    \caption{MSE performance of GD algorithms in ridge regression ($\eta\!=\!158.48$) of crime data.
    For CHGD, $T$ was set to $32$.
}
    \label{fig_rid}
\end{figure}

{Figure~\ref{fig_rid} shows the MSE performance between ridge estimation and estimates of the GD algorithms as a function of the iteration index $t$ when $\eta=158.48$.
In CHGD, we set $T=32$ and the order of  the Chebyshev steps was properly permuted (see Appendix F).    
We found that the momentum method takes a small number of iteration steps
 to exhibit better MSE performance than GD, although its convergence speed was much faster for large $t$.
Although the MSE of CHGD formed a wavy shape, the estimates at every $T=32$ steps were reasonably accurate and converged quickly.}
We thus found that CHGD exhibited fast convergence compared with the other algorithms.
{CHGD requires no momentum terms and thus less computational resources, which would be advantageous for a high-dimensional problem. For example, CHGD will be useful to solve a linear equation involving a large sparse covariance matrix in Gaussian process~\cite{williams2006gaussian}.}

\section{Concluding remarks}
 
 In this paper, we studied a nontrivial learned step size sequence that appeared in {DUGD.
 We proved} that minimizing the MSE loss in DUGD reduced the spectral radius related to the convergence rate.
{Additionally, we introduced learning-free Chebyshev steps that minimized the 
the upper bound of the spectral radius. 
We showed that the Chebyshev steps accelerated the convergence speed compared with a naive GD, and the rate approached the strict lower bound of first-order methods in a specific limit.
 The numerical results supported the analyses and showed that the Chebyshev steps reproduced the learned step size sequence in {DUGD, which provides a plausible interpretation of the learned parameters. }
 Moreover, CHGD exhibited a reasonable convergence speed compared with other accelerated GD algorithms, although it did not require a momentum term.
 }

 There are several {open problems}.
{One is the extension of the analysis in this paper to convex and non-convex problems other than quadratic convex problems. 
Another is the application of the Chebyshev steps to other GD-based algorithms, such as stochastic GD. }
For example, the idea of the Chebyshev steps is successfully applicable to the fixed-point iteration~\cite{wadayama2020chebyshev} and Landweber algorithm~\cite{wadayama2020chebyshev2}. 


\section*{Acknowledgement}
The authors warmly thank Mr. S. Khobahi for useful comments on the manuscript.
This work was partly supported by 
JSPS Grant-in-Aid for Scientific Research (B) Grant Number 19H02138 (TW), 
JSPS Grant-in-Aid for Early-Career Scientists Grant Number 19K14613 (ST), 
and the Telecommunications Advancement Foundation (ST).

\appendices
\section{Experimental setting for Figure 1}\label{app1}

We describe the experimental setting for Figure 1 in the main text.

We consider the noiseless measurement $\bm y=\bm H \bm \beta$, where the measurement matrix $\bm H\in \mathbb{R}^{m\times n}$ is the Gaussian random matrix whose elements are i.i.d. Gaussian random variables 
with zero mean and variance $1/n$.
We assume that $(n,m)=(300,600)$ and each element of $\bm \beta$ follows the normal distribution.
DUGD is defined by 
\begin{equation}
\bm{\beta}^{(t+1)}= \bm \beta^{(t)}{+}\gamma_t \bm H^T(\bm y - \bm H \bm \beta^{(t)}) \: (t=0,\dots, T-1),
\label{eq_GD2}
\end{equation}
where $T$ is the total number of iterations (or layers) and  $\{\gamma_t\}_{t=0}^{T-1}$ are trainable step size parameters. The initial point is given by $\bm \beta^{(0)}=\bm 0$.

DUGD was implemented using PyTorch 1.3~\cite{NIPS2019_9015}. 
Each training data were given by a pair consisting of the true solution $\bm{\tilde\beta}$ and the corresponding measurement vector $\bm{\tilde y}$ for a given $\bm H$.
The training process of DUGD was conducted using incremental training in which we gradually increased the number of layers (iterations $T$) by initializing the value of the parameters $\{\gamma_t\}$ ($t=0,\dots T-1$) according to the learned values in the former training process (generation) to improve the performance of DUGD.
At the beginning of the training process, all the initial values of  $\{\gamma_t\}$ were set to $0.3$. 
In each generation, the parameters were optimized using Adam optimizer~\cite{kingma2014adam} with a learning rate of $0.002$ to minimize the MSE loss function $\|\bm{\tilde \beta}-\bm\beta^{(T)}\|_2^2/n$ between the output of  DUGD using $\bm{\tilde{y}}$ and the true solution $\bm{\tilde{\beta}}$.
As mini-batch training, $500$ mini-batches of size $200$ were used in each generation.
{The MSE was evaluated as a generalization error {over $10^4$ random samples}.}

\section{Proof of Theorem 3.2}\label{app2}

\begin{proof}
As the matrix $\bm A$ is a Hermitian matrix, it can be decomposed by $\bm A = \bm U\bm \Lambda \bm U^\ast$ using the unitary matrix $\bm U$ and diagonal matrix $\bm \Lambda:=\mathrm{diag}(\lambda_1,\dots,\lambda_n)$ with eigenvalues $\lambda_1,\dots,\lambda_n$ of $\bm A$.
Then, we have 
\begin{align}
\mathsf{E}_{\bm x^{(0)}}L(\bm x^{(T)})
&= \frac{1}{n} \mathsf{E}_{\bm x^{(0)}}\left\|\prod_{t=0}^{T-1}(\bm I_n-\gamma_t \bm U\bm \Lambda \bm U^\ast) \bm x^{(0)}\right\|_2^2\nonumber\\
&= \frac{1}{n}  \mathsf{E}_{\bm x^{(0)}}\left\|\bm U\left(\prod_{t=0}^{T-1}(\bm I_n-\gamma_t \bm \Lambda )\right) \bm U^\ast\bm x^{(0)}\right\|_2^2\nonumber\\
&= \frac{1}{n}  \mathsf{E}_{\bm x^{(0)}}\left\| \left(\prod_{t=0}^{T-1}(\bm I_n-\gamma_t \bm \Lambda )\right) \bm U^\ast\bm x^{(0)}\right\|_2^2\nonumber\\
&:=  \frac{1}{n}  \mathsf{E}_{\bm x^{(0)}}\left\| \bm D^{(T)} \bm U^\ast\bm x^{(0)}\right\|_2^2
\end{align}
where $\bm{D}:=\mathrm{diag}(\prod_{t=0}^{T-1} (1-\gamma_t \lambda_i))$ is the diagonal matrix whose $(i,i)$-element is an eigenvalue of $\bm{Q}^{(T)}$ corresponding to the eigenvalue $\lambda_i$ of $\bm{A}$.
In the last line, we use $\bm U\bm U^\ast=I_n$.
Introducing the column vectors of $\bm U$ by  $\bm{U} := (\bm u_1, \dots,\bm u_n)$, 
$ \bm U^\ast\bm x_0 = (\bm u_1^\ast \bm x_0, \dots,\bm u_n^\ast \bm x_0)^T$ holds.

Using $j:=\mathrm{argmax}_{i}|\prod_{t=0}^{T-1}(1-\gamma_t \lambda_i)|$, we have
\begin{align}
L(\bm x^{(T)})
&= \frac{1}{n}  \mathsf{E}_{\bm x^{(0)}}\left\|\sum_{i=1}^n \bm D^{(T)}_{i,i} \bm u_i^\ast\bm x^{(0)}\right\|_2^2\nonumber\\
&\ge \frac{1}{n}  \left|\prod_{t=0}^{T-1}(1-\gamma_t \lambda_j) \right|^2
\mathsf{E}_{\bm x^{(0)}}\left\|\bm u_j^\ast\bm x^{(0)}\right\|_2^2\nonumber\\
&:= \frac{C'}{n}  \left|\prod_{t=0}^{T-1}(1-\gamma_t \lambda_j) \right|^2,
\end{align}
where $C':= \mathsf{E}_{\bm x^{(0)}}\|\bm u_j^\ast\bm x^{(0)}\|^2(<\infty)$ is a positive constant because the probability density function of $\bm{x}^{(0)}$ is assumed to be isotropic.
Recalling that $\rho(\bm Q^{(T)})=\max_{i}\left|\prod_{t=0}^{T-1}(1-\gamma_t\lambda_i)\right|=\left|\prod_{t=0}^{T-1}(1-\gamma_t \lambda_j) \right|$ from 
Lemma 3.1, we have $\mathsf{E}_{\bm x^{(0)}}L(\bm x^{(T)}) \ge C'\rho(\bm Q^{(T)})^2/n$, 
which is identical to (8).
 \end{proof}

The proof indicates that the constant $C$ is explicitly given by $C= (\mathsf{E}_{\bm x^{(0)}}\|\bm u_j^\ast\bm x^{(0)}\|^2)^{-1/2}$ with $j:=\mathrm{argmax}_{i}|\prod_{t=0}^{T-1}(1-\gamma_t \lambda_i)|$.
In the special case in which each element of $\bm{x}^{(0)}$ is an i.i.d. random variable, $C$ can be easily calculated. 
This fact is used in the numerical experiment in Section 4.2. In this case, we have $C=1/\sqrt{2}$ because each element of $\bm{x}^{(0)}$ follows the Gaussian distribution with unit mean and unit variance.

\section{Proof of Theorem 3.4}\label{app3}

Before we provide the proof of Theorem 3.4, we prove the following lemma related to the minimax property of  Chebyshev polynomials.
The Chebyshev polynomial $C_n(x)$ of degree $n$ ($n=0,1,\dots$) is defined as a recursive relation $C_{n+1}(x):=2x C_{n}(x)-C_{n-1}(x)$ with $C_0(x):=1$ and $C_1(x):=x$. Let $C[a,b]$ ($a<b$) be the Banach space defined as 
$\ell_\infty$-norm, i.e., $\|f\|:=\max_{x\in[a,b]}|f(x)|$.

\begin{lemma}\label{lem_spec}
Suppose that $b>a>0$.
Let $D\subset C[a,b]$ be a subspace of polynomials of $z$ on $[a,b]$ represented by 
$\prod_{k=0}^{n-1}(1-\alpha_k z)$ for any $\alpha_0,\cdots,\alpha_{n-1}\in \mathbb{R}$.
We define the Chebyshev steps of length $n$ as
\begin{equation}
\gamma_k := \left[\frac{a+b}{2}+\frac{b-a}{2}\cos\left(\frac{2k+1}{2n}\pi\right)\right]^{-1} \:(k=0,1,\dots,n-1),
\label{eq_chstep}
\end{equation} 
and a normalized Chebyshev polynomial $\hat\varphi(z)$ of degree $n$ as
\begin{equation}
\hat \varphi(z) 
:=  \frac{C_n\left(\frac{2z-a-b}{b-a}\right)}{C_n\left(-\frac{a+b}{b-a}\right)}. \label{eq_norm_ch}
\end{equation}
Then, the following statements hold.

(a) The function $\hat \varphi:[a,b]\rightarrow \mathbb{R}$ belongs to $D$ as a result of setting $\alpha_k=\gamma_k^{-1}$ ($k=0,1,\dots,n-1$).

(b) The function $\hat \varphi:[a,b]\rightarrow \mathbb{R}$ is a polynomial in $D$ that minimizes the norm $\|\cdot\|$.
\end{lemma}

\begin{proof}
(a) The Chebyshev polynomial $C_n(x)$ of degree $n$ has $n$ zeros in $(-1,1)$, which are given by
$x_k=\cos((2k+1)\pi/(2n))$ ($k=0,\dots,n-1$)~\cite[Section 2.2]{mason2002chebyshev}; that is, $C_n(x)= \prod_{k=0}^{n-1}(x-{x}_k)$ holds.
Then, using the affine transformation from $[a,b]$ to $[-1,1]$, we have
\begin{equation}
C_n\left(\frac{2z-a-b}{b-a}\right) =\prod_{k=0}^{n-1}\left(z-\frac{1}{\gamma_k}\right). \label{eq_cn_zero}
\end{equation}
Because $C_n(-(a+b)/(b-a)) =\prod_{k=0}^{n-1}(-\gamma_k^{-1})$,
we have $\hat \varphi(z) = \prod_{k=0}^{n-1}(1-\gamma_k z)$, which indicates that the statement holds.

(b) We show that $\hat\varphi(z)$ is a minimizer of $\|\cdot\|$ among functions in $D$ by indirect proof.
Assume that there exists $\tau(z)\in D$ of {at most} degree $n$ except for $\hat\varphi(x)$ satisfying $\|\hat\varphi(z)\|>\|\tau(z)\|$.
As $\check x_k:=\cos(k\pi /n)$ ($k=0,\dots,n$) are the extreme points in $[-1,1]$ (including both edge points) of $C_n(x)$ ~\cite[Section 2.2]{mason2002chebyshev},
$\check z_k := (a+b)/2 + (b-a)\check x_k/2(\in [a,b])$ are those of $\hat\varphi(z)$.
Particularly, the sign of the extremal value at $\check x_k$ (or $\check z_k$) changes alternatively; 
$C_n(\check x_n)=1$, $C_n(\check x_{n-1})=-1$, $C_n(\check x_{n-2})=1$, and so on (or 
$\hat\varphi(\check z_n)=\varphi_0$, $\hat\varphi(\check z_{n-1})=-\varphi_0$, $\hat\varphi(\check z_{n-2})=\varphi_0$, and so on when $\varphi_0:=1/C_n(-(a+b)/(b-a))$) hold~\cite[Lemma 3.6]{mason2002chebyshev}.
The assumption indicates that $n+1$ inequalities, $\tau(\check z_{n})< \varphi_0$, $\tau(\check z_{n-1})> -\varphi_0$, $\tau(\check z_{n-2})< \varphi_0$, and so on hold; that is, a polynomial $\delta(z):=\tau(z)-\hat\varphi(z)$ of degree at most $n$ has $n$ zeros in $[a,b]$.

However, because $\tau(z), \hat\varphi(z)\in D$, the constant term of $\delta(z)$ is equal to zero, which  suggests  that $\delta(z)$ has at most $n-1$ zeros in $[a,b]$.
This results in a contradiction of the assumption and shows that  $\hat \varphi:[a,b]\rightarrow \mathbb{R}$ minimizes the norm $\|\cdot\|$ in $D$.  
\end{proof}

It is straightforward to prove Theorem 3.4 from this lemma.  
\begin{proof}[proof of Theorem 3.4]
Using the notation of Lemma~\ref{lem_spec}, we notice that $a=\lambda_1$, $b=\lambda_n$, and
\begin{equation}
 \rho^{\mathrm{upp}}(\bm Q^{(T)}) = \max_{ \lambda\in [\lambda_1,\lambda_n] }\left|\prod_{t=0}^{T-1}(1-\gamma_t\lambda)\right| 
= \left\|\prod_{t=0}^{T-1}(1-\gamma_t\lambda) \right\|.
\end{equation}
From Lemma~\ref{lem_spec}, the Chebyshev steps of length $T$ form a sequence that minimizes $\rho^{\mathrm{upp}}(\bm Q^{(T)})$. 
\end{proof}

\section{Proof of Theorem 3.6}\label{app4}


\begin{proof}
To simplify the notation, we use $\kappa:= \lambda_n/\lambda_1(>1)$ as a condition number of the matrix $\bm A$.
From (10) and Lemma~\ref{lem_spec}, we have
\begin{align}
\rho(\bm Q^{(T)}_{\mathrm{Ch}})&\le
\rho^{\mathrm{upp}}(\bm Q^{(T)}_{\mathrm{Ch}})\nonumber\\
&= \max_{x\in[\lambda_1,\lambda_n]}\hat\varphi(x) \nonumber\\
&= \left|C_T\left(-\frac{\kappa+1}{\kappa-1}\right)\right|^{-1} \nonumber\\
&= \left|(-1)^T\frac{(\sqrt\kappa+1)^{2T}+(\sqrt{\kappa}-1)^{2T}}{2(\kappa-1)^T}\right|^{-1}\nonumber\\
&= \left\{\frac{1}{2}\left[ \left(\frac{\sqrt\kappa+1}{\sqrt\kappa-1}\right)^T
+\left(\frac{\sqrt\kappa-1}{\sqrt\kappa+1}\right)^T\right]\right\}^{-1}. \label{eq_qt}
\end{align}
We use the properties that $|C_n(x)|\le 1$ holds for $\forall x\in[-1,1]$ and 
$C_n(x)=[(x+\sqrt{x^2-1})^n+(x-\sqrt{x^2-1})^n]/2$ holds for $|x|>1$, 
and the identity 
$x\pm \sqrt{x^2-1}=(\sqrt \kappa \pm 1)^2/(\kappa-1)$ when $x=(\kappa+1)/(\kappa-1)$.

By contrast, the spectral radius when the optimal constant step size $\gamma_t^\ast = 2/(\lambda_1+\lambda_n)$ is used can be calculated directly.
We have
\begin{equation}
\rho(\bm Q^{(T)}_{\mathrm{s}})= \prod_{t=0}^{T-1} \max_{i} \left|1-\gamma^\ast \lambda_i \right|
=\left(\frac{\kappa-1}{\kappa+1} \right)^T.
\end{equation}

Finally, we show that $\rho^{\mathrm{upp}}(\bm Q^{(T)}_{\mathrm{Ch}})< \rho(\bm Q^{(T)}_{\mathrm{s}})$ holds.
This is equivalent to the following relation for $\kappa>1$:   
\begin{align}
&\frac{1}{2}\left[ \left(\frac{\sqrt\kappa+1}{\sqrt\kappa-1}\right)^T
+\left(\frac{\sqrt\kappa-1}{\sqrt\kappa+1}\right)^T\right] 
-\left(\frac{\kappa+1}{\kappa-1}\right)^T \nonumber\\
&=\frac{(\sqrt\kappa+1)^{2T}+(\sqrt\kappa-1)^{2T}-2(\kappa+1)^{T}}{2(\kappa-1)^{T}}> 0.    
\label{eq_in11}
\end{align}
If we set $X:=\sqrt \kappa (>1)$, the $(2t)$th coefficient of $(X+1)^{2T}/2+(X-1)^{2T}/2- (X^2+1)^T$
is given by $\binom{2T}{2t}-\binom{T}{t}$. Additionally, its coefficients of odd orders are equal to  zero. 
From the Vandermonde identity:
\begin{equation}
\binom{m+n}{r} = \sum_{k=0}^{r}\binom{m}{k}\binom{n}{r-k},
\end{equation}
we find 
\begin{equation}
\binom{2T}{2t} = \sum_{l=0}^{2t}\binom{T}{l}\binom{T}{2t-l}\ge \binom{T}{t}^2\ge \binom{T}{t},
\end{equation}
(equality holds only when $t=0,T$), which indicates that (\ref{eq_in11}) holds.

We thus prove that $\rho(\bm Q^{(T)}_{\mathrm{Ch}})\le \rho^{\mathrm{upp}}(\bm Q^{(T)}_{\mathrm{Ch}})< \rho(\bm Q^{(T)}_{\mathrm{s}})$ when $\lambda_1<\lambda_n$.
\end{proof}

\section{Order of the learned step sizes}\label{app5}

In this subsection, we describe how to determine a permutation of the Chebyshev steps that 
reproduces a zig-zag pattern of the learned step size parameters of DUGD.
A key observation is the dynamics of trainable step size parameters in the training process.

Figure~\ref{fig_zig} shows the dynamics of the trainable step size parameters. During incremental training, the number of trainable parameters increases at every $2000$ mini batches; that is, DUGD of $T$ iterations (or layers) is trained from the $2000(T-1)$th mini batches to the $2000T$th mini batches, which we call the $T$th generation of incremental training.  
After the $T$th generation ends, we add an initialized parameter $\gamma_{T+1}$ to the learned parameters $\{\gamma_{t}\}_{t=0}^T$ to start a new generation.
As we can see in Figure~\ref{fig_zig}, the trainable parameters immediately move toward a (sub)optimal point to reduce the MSE loss function, which forms a staircase shape of $\gamma_t$. 
Although there are $T!$ optimal points of step sizes by permutation symmetry, it is numerically suggested that DUGD chooses an optimal point so that the ``distance'' (discussed later) from the former learned parameters is minimized. 
This seems natural because these parameters are updated by a GD-based optimizer and the convergent point depends on the initial point.

\begin{figure}[t]
   \centering
   \includegraphics[width=0.98\hsize]{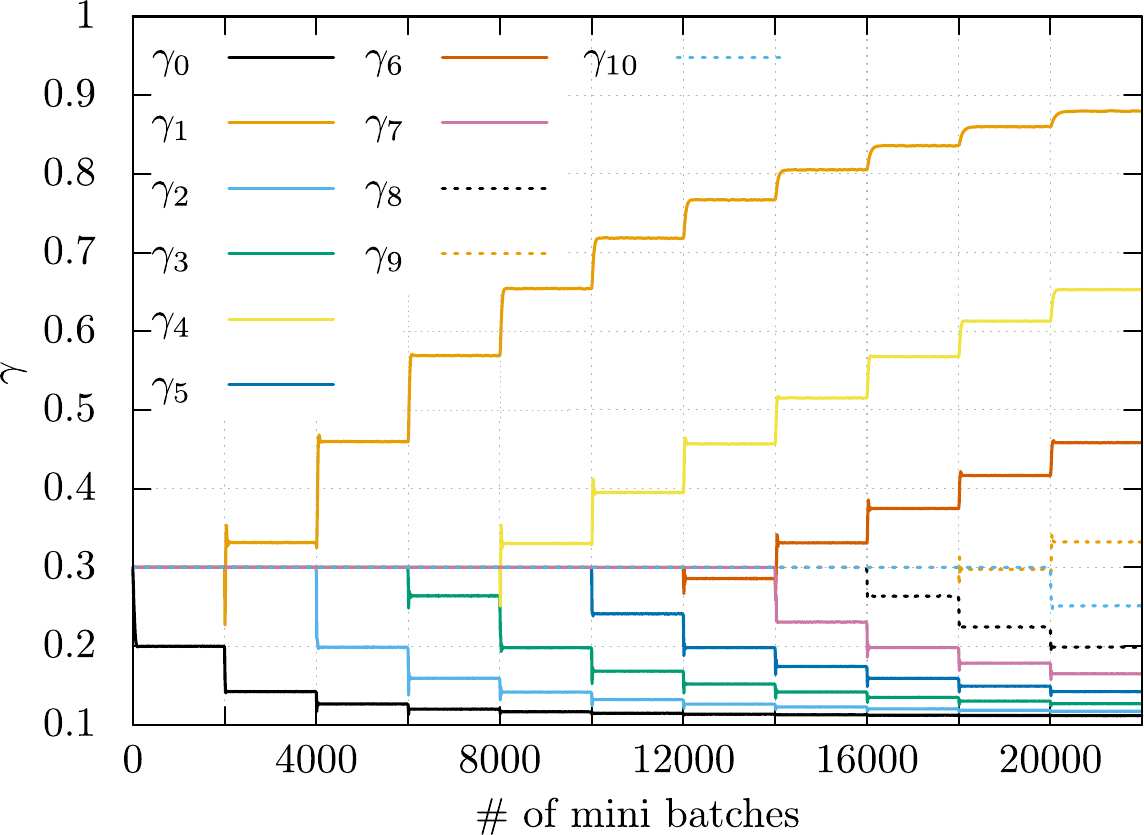}
    \caption{Dynamics of  step size parameters $\{\gamma_t\}_{t=0}^{10}$ of DUGD ($T=11$) when $(n,m)=(300,1200)$. Horizontal line represents the number of mini batches in incremental training. All the initial values were set to $0.3$. In incremental training, the number of learning parameters increased by one for every $2000$ mini batches  fed to DUGD. For example, only $\gamma_0$ (black) was trained during the first $2000$ mini batches and
 $\gamma_0$ and  $\gamma_1$ (orange) were trained during the next $2000$ mini batches.     
}
    \label{fig_zig}
\end{figure}

From these observations, we attempt to emulate the order of the learned step size parameters using the Chebyshev steps.
We consider a training process of DUGD that minimizes the spectral radius $\rho(\bm Q^{(T)})$ instead of the MSE loss function. 
Although it seems practically difficult, we assume that we obtain the Chebyshev steps as an approximate solution.
The problem is which order of the Chebyshev steps is chosen at each generation.
We thus determine an order of the Chebyshev steps that minimizes the ``distance'' from a given initial point to a point whose elements are permuted Chebyshev steps. 
As a measure of distance, we use a simple Euclidean norm because an actual distance defined by an energy landscape is difficult to calculate.
The details of the algorithm are shown in Algorithm~\ref{alg_0}. 
To emulate incremental training, the length of the Chebyshev steps is gradually increased. As an initial point ($\gamma_1,\dots,\gamma_{t+1}$) of length $t+1$,  $\gamma_1, \dots, \gamma_t$ are set to the optimally permuted Chebyshev steps in the last generation and $\gamma_{t+1}$ is set to a given initial value.
Then, an optimal permutation of the Chebyshev steps of length $t+1$ is searched so that its distance from the initial point takes the minimum value. The point is used as an initial point of the next generations. 
As shown in Figure 6, this successfully reproduces the zig-zag shape of learned step size parameters that depends on an initial value of $\gamma_t$.

\begin{algorithm}[t]
   \caption{Emulation of incremental training using the Chebyshev steps}
   \label{alg_0}
\begin{algorithmic}
   \STATE {\bfseries Input:} maximum eigenvalue $\lambda_1$, minimum eigenvalue $\lambda_n$, number of iterations $T$, initial value $u$
   \STATE Initialize $\bm{c} = (2/(\lambda_1+\lambda_n))$
   \FOR{$t=2$ {\bfseries to} $T$}
   \STATE Set $v$ to a sufficiently large number	   
   \STATE $\bm{d} = (\bm c, u) $ 
   \STATE Define $\bm c$ as Chebyshev steps of length $t$ for $\lambda_1$ and $\lambda_n$
   \FOR{$\pi$ {\bfseries to all possible permutations} $\Pi(t)$}
   \STATE Define $\bm P_{\pi}$ as the permutation matrix of $\pi$
   \IF{$v>\|\bm{d}- \bm{P}_{\pi}\bm{c}\|_2$}
   \STATE $v=\|\bm{d}- \bm{P}_{\pi}\bm{c}\|_2$, $\bm P=\bm P_{\pi}$
   \ENDIF
   \ENDFOR
   \STATE $\bm c=\bm P\bm c$
   \ENDFOR
   \STATE {\bfseries Return:} $\bm c$

\end{algorithmic}
\end{algorithm}

\section{Order optimization of the Chebyshev steps}\label{app6}

\begin{algorithm}[tb]
   \caption{Permutation search}
   \label{alg_2}
\begin{algorithmic}
   \STATE {\bfseries Input:} maximum eigenvalue $\lambda_n$, minimum eigenvalue $\lambda_1$, number of iterations $T:=2^s$ ($s\in \mathbb{N}$)
   \STATE Set $v$ to a sufficiently large number	   
   \STATE Define $\bm c$ as the Chebyshev steps of length $T$ for $\lambda_1$ and $\lambda_n$.
   \FOR{$(a,b,c)$ satisfying $1\le a,b,c \le T-1$, $a\equiv 1$ (mod $4$), and $b$: odd}
   \STATE Define $\bm P$ as the permutation matrix corresponding to $(a,b,c)$
   \IF{$v>\rho_{\mathrm{temp}}(T)$}
   \STATE $v=\rho_{\mathrm{temp}}(T)$, $\bm Q=\bm P$
   \ENDIF
   \ENDFOR
   \STATE {\bfseries Return:} $\bm P\bm c$, $(a,b,c)$
\end{algorithmic}
\end{algorithm}

The order of the Chebyshev steps is important practically to ensure numerical stability.
Figure~\ref{fig_order} shows the MSE performance of CHGD with different orders of the Chebyshev steps. 
We found that using an ascending order led to a search point with huge values that might cause a digit loss.
To avoid this instability, we need to permute the step size parameter sequence.
{It is noted that the performance of CHGD itself is ensured every $T$ iterations if numerical errors are ignorant.}
Because the total number of permutations rapidly diverges depending on $T$, we focus on permutations defined by
\begin{equation}
\pi(t+1) \equiv a\pi(t) + b \quad (\mbox{mod } T), 
\end{equation}
 where $\pi(0):=c\in \{0,1,\dots,T-1\}$ is an initial index.
Then, the sequence $\{\pi(t)\}_{t=0}^{T-1}$ is a permutation of $\{0,1,\dots,T-1\}$ 
if $b$ is odd, $a\equiv 1$ (mod $4$), and $T=2^s$ ($s\in\mathbb{N}$).
 We then search parameters $a,b,c$ that minimizes the maximum temporal spectral radius of $\bm{W}^{(t)}$, that is,
\begin{equation}
\rho_{\mathrm{temp}}(T):= \max_{t\in\{0,1,\dots, T-1\}}  \left(\max_{\lambda\in[\lambda_1,\lambda_n]} \left|\prod_{t'=0}^{t}(1-\gamma_{t'}\lambda )\right|\right).
\end{equation}
Algorithm~\ref{alg_2} shows the pseudocode of the searching algorithm.

\begin{figure}[t]
   \centering
   \includegraphics[width=0.65\hsize]{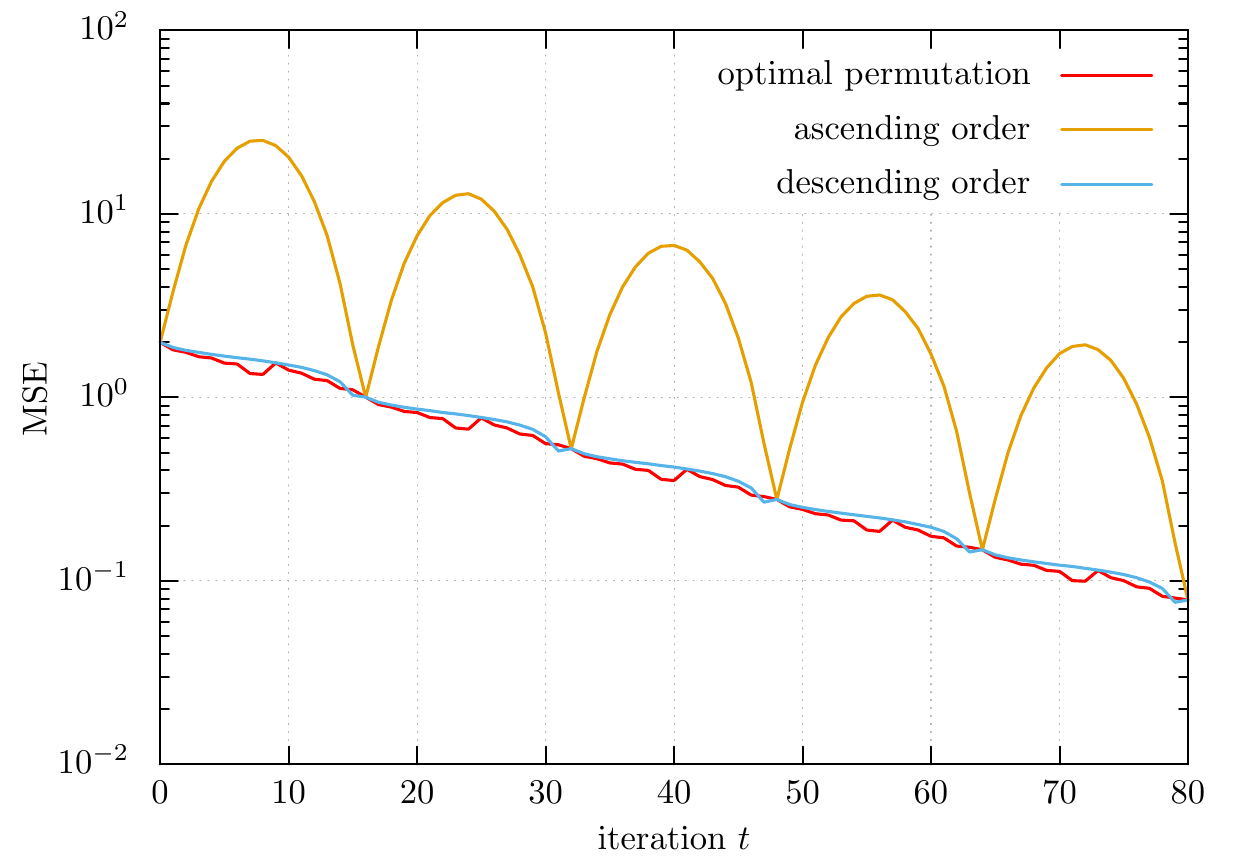}
    \caption{MSE performance of CHGD ($T=16$) with optimal permutation (red), ascending order (orange), and descending order (blue) when $(n,m)=(500,800)$. They have the same MSE every $T=16$ iterations.
}
    \label{fig_order}
\end{figure}

Figure~\ref{fig_order} shows the MSE performance of CHGD ($T=16$) with and without permutation when $n=500$ and $m=800$. The MSE is evaluated using $200$ samples.
In this case, the asymptotic value of the condition number was $\kappa=8.54$ and  the optimal parameters are given by $(a,b,c)=(1,9,7)$.
In CHGD without permutation, the step sizes was given by (\ref{eq_chstep}) in a descending or ascending manner. Particularly, GD with ascending Chebyshev steps had a relatively large MSE. 
The numerical results show that CHGD with optimal permutation decreased the MSE effectively.
  
An example of $(a,b,c)$ for different $T$, $\lambda_1=1$, and $\lambda_n=\kappa$
is given in Table~\ref{tab_1}.
We found that the optimal choice of $(a,b,c)$ depended on $\lambda_1$ and $\lambda_n$. 
In Section 5.2, the Chebyshev steps were permuted according to $(a,b,c)=(1,11,10)$.

\begin{table}[tb]
\caption{Searched permutation parameters $(a,b,c)$ of Chebyshev steps when $\lambda_1=1$ and $\lambda_n=\kappa$.}
\label{tab_1}
\vskip 0.15in
\begin{center}
\begin{small}
\begin{sc}
\begin{tabular}{c|cccc}
\toprule
 & $T=8$ & $T=16$ & $T=32$ \\
\midrule
$\kappa=4$   & $(1,5,3)$& $(1,9,7)$& $(1,17,15)$\\
$\kappa=16$ & $(1,5,3)$ & $(1,9,7)$& $(1,17,15)$\\
$\kappa=64$    &$(1,3,2)$ & $(1,9,7)$& $(1,17,15)$ \\
$\kappa=128$    &$(1,3,2)$ & $(13,3,6)$& $(1,17,15)$        \\
\bottomrule
\end{tabular}
\end{sc}
\end{small}
\end{center}
\vskip -0.1in
\end{table}


\end{document}